\pdfoutput=1
\documentclass[11pt,a4paper]{article}
\usepackage[utf8]{inputenc}
\usepackage{amsfonts, amsmath, amssymb, color, a4wide}
\usepackage{hyperref}
\usepackage[utf8]{inputenc} % allow utf-8 input
\usepackage[T1]{fontenc}    % use 8-bit T1 fonts
\usepackage{url}            % simple URL typesetting
\usepackage{booktabs}       % professional-quality tables
\usepackage{amsfonts}       % blackboard math symbols
\usepackage{nicefrac}       % compact symbols for 1/2, etc.
\usepackage{microtype}      % microtypography

\usepackage{amsfonts}
\usepackage{amsmath}
\usepackage{amssymb}
\usepackage{bm}
\usepackage{algorithm}
\usepackage{algorithmic}
\usepackage{mathtools}
\usepackage{xspace}
\usepackage{xcolor}
\usepackage{amsthm}

\usepackage{natbib}
\RequirePackage[shortlabels]{enumitem}

\newtheorem{theorem}{Theorem}[section]

\newtheorem{proposition}[theorem]{Proposition}
\newtheorem{lemma}[theorem]{Lemma}
\newtheorem{corollary}[theorem]{Corollary}

\newtheorem{claim}[theorem]{Claim}

\newtheorem{assumption}{Assumption}

\newcommand{\selectproc}{\textbf{Select}\xspace}
\newcommand{\tryselectproc}{\textbf{Try-Select}\xspace}
\newcommand{\optimproc}{\textbf{Optim}\xspace}
\newcommand\R{\mathbb{R}}

\input{macros.tex}

\newcommand{\vertiii}[1]{{\left\vert\kern-0.25ex\left\vert\kern-0.25ex\left\vert #1 
		\right\vert\kern-0.25ex\right\vert\kern-0.25ex\right\vert}}

\newcommand{\selarm}{\hat{i}}
\newcommand{\bestarm}{i^*}

\DeclareMathOperator*{\argmin}{arg\,min}

\title{Online Orthogonal Matching Pursuit}
\author{El Mehdi Saad\footnote{Laboratoire de Math\'ematiques d'Orsay, Univ.\ Paris-Sud, Universit\'e Paris-Saclay.} , Gilles Blanchard$^*$, Sylvain Arlot$^*$}
\date{}
\begin{document}
\maketitle

\begin{abstract}
Greedy algorithms for feature selection are widely used for recovering sparse high-dimensional vectors in linear models. In classical procedures, the main emphasis was put on the sample complexity, with little or no consideration of the computation resources required. We present a novel online algorithm: Online Orthogonal Matching Pursuit (OOMP) for online support recovery in the random design setting of sparse linear regression. Our procedure selects features sequentially, with one pass over data,
alternating between allocation of samples only as needed to candidate features, and optimization over the selected set of variables to estimate the regression coefficients. Theoretical guarantees about the output of this algorithm are proven and its computational complexity is analysed.
\end{abstract}

\section{Introduction}

In the context of large scale machine learning, one often deals with massive data-sets and a considerable number of features. While processing such large data-sets, one is often faced with scarce computing resources.
The adaptability of online learning algorithms to such constraints made them very popular in the machine learning community.

In the current work we address the problem of online feature selection, i.e support recovery algorithms restricted to a single training pass over the available data. This setting is particularly relevant when the system cannot afford several passes throughout the training set: for example, when dealing with massive amounts of data or when memory or processing resources are restricted, or when data is not stored but presented in a stream.

%In this article, we consider the classical problem of exact support recovery for the sparse linear regression. We present a novel online feature selection algorithm, assuming that we dispose of an abundant amount of data but limited computing resources.

Suppose that there exists a vector $\beta^{*}\in \mathbb{R}^{d}$ with $\|\beta^{*}\|_0 =s^{*} \le d$ such that the response variable $y$ is generated according to the linear model $y = \langle x,\beta^{*}\rangle+\epsilon$, where $\epsilon$ satisfies $\mathbb{E}[\epsilon | x]=0$, let $S^* = \text{supp}\left(\beta^* \right)$. Throughout the article, we consider that the feature vector $x$ is random, and we assume that $|y|<1$ and $\|x\|_{\infty} <M$ almost surely for a known constant $M>0$. The straightforward formulation of sparse regression using a $l_{0}-$ pseudo-norm constraint is computationally intractable. This challenge motivated the rise of many computationally tractable procedures
whose statistical validity has been established under additional assumptions such as the Irrepresentable Condition (IC) and Restricted Isometry Property (RIP).

%Many algorithms have been proposed for support recovery, the most popular procedures use a convex relaxation with the $l_1-$norm (LASSO based algorithms, \cite{tibshirani1996regression}), and greedy procedures such as
%Orthogonal Matching Pursuit algorithm (OMP, \cite{mallat1993matching}), where features are selected sequentially. In this paper, we develop a novel online variant of OMP. Theoretical guarantees about OMP on support recovery were developed by \cite{zhang2011sparse}, under the IC+RIP assumption, and many variants have been developed \cite{blumensath2008gradient,combettes2019blended}, where different optimization procedures are used instead of ordinary least squares. However, the computational complexity remains of the order $\mathcal{O}(nd)$ for one variable selection step and $\mathcal{O}(s^*nd)$ for total support recovery, with a sample size satisfying $n =\Omega \paren{ \max\paren[1]{ s^*, \frac{1}{\min\{ |\beta^{*}_i|^{2}, \beta^{*}_i \neq 0\}}} }$ for exact support recovery with a high probability guarantee. A drawback of these procedures, besides the need to perform multiple passes over the training set, is that the sample size, hence the computational complexity of every step, depends on $(\min\{ |\beta^{*}_i|, \beta^{*}_i \neq 0\})^{-1}$. This is unpractical, since a prior knowledge on the order of the smallest regression coefficient is necessary for guaranteed support recovery, moreover intuition suggests that recovery of the larger coefficients of $\beta^*$ should be possible with less data and hence less computational complexity.

Many algorithms have been proposed for support recovery, the most popular procedures use a convex relaxation with the $l_1-$norm (LASSO based algorithms, \cite{tibshirani1996regression}), and greedy procedures such as
Orthogonal Matching Pursuit algorithm (OMP, \cite{mallat1993matching}), where features are selected sequentially. In this paper, we develop a novel online variant of OMP. Theoretical guarantees about OMP on support recovery were developed by \cite{zhang2011sparse}, under the IC+RIP assumption, and many variants have been developed \cite{blumensath2008gradient,combettes2019blended}, where different optimization procedures are used instead of ordinary least squares. However, the computational complexity remains of the order $\mathcal{O}(nd)$ for one variable selection step and $\mathcal{O}(s^*nd)$ for total support recovery, with a sample size satisfying $n =\Omega \paren{ \max\paren[1]{ s^*, \frac{1}{\min\{ |\beta^{*}_i|^{2}, \beta^{*}_i \neq 0\}}} }$ for exact support recovery with a high probability guarantee. A drawback of these procedures, besides the need to perform multiple passes over the training set, is that the sample size, hence the computational complexity of every step, depends on $(\min\{ |\beta^{*}_i|, \beta^{*}_i \neq 0\})^{-1}$. Intuition suggests that recovery of the larger coefficients of $\beta^*$ should be possible with less data and hence less computational complexity.
We propose a feature selection procedure that is consistent with this intuition.

If the support size $s^*$ is known, the proposed algorithm (OOMP) halts after recovering all features in $S^*$. Otherwise, it relies on some external criterion (such as a runtime budget), whenever halted, the procedure returns a set of features guarantees to belong to $S^*$ with high probability. Moreover, we show that support recovery is achieved in finite time and provide a control on the computational complexity necessary to attain this goal.

\subsection{Main contributions}

This paper is about the design and analysis of support recovery for linear models in the online setting. We make the following contributions:
\begin{itemize}
	\item We design a general modular procedure, where the learner can use any black-box optimization algorithm combined with an approximate best arm identification approach, provided those procedures come with suitable guarantees. We show that at any interruption time, it is guaranteed with high probability that the set of selected features $S$ satisfies: $S \subseteq S^*$.
	\item We instantiate the general design using a variant of the stochastic gradient descent for the optimization and a LUCB-type (Lower Upper Confidence Bound) procedure for approximate best arm selection. The proposed algorithm has the advantage of being adapted to the streaming setting (i.e. requiring only one pass over data).
	\item A prior knowledge on the support size $s^*$ or the magnitude of the smallest coefficient: $\min\{ |\beta^{*}_i|, \beta^{*}_i \neq 0\}$, is not necessary to run the procedure. We show that OOMP recovers the support $S^*$ in finite time and provide a control on the runtime necessary to achieve this objective.
	\item We compare the runtime required for support recovery using OOMP ($C^{\text{OOMP}}$) with the corresponding runtime using batch version OMP  ($C^{\text{OMP}}$). We show that when $d>(s^*)^3$, it always holds $C^{\text{OOMP}} = \mathcal{O}(C^{\text{OMP}}\log^2\left(C^{\text{OMP}}\right))$, and when the coefficients of $\beta^*$ have a different order of magnitude, $C^{\text{OOMP}}$ can be much smaller than $C^{\text{OMP}}$. We provide some examples (such as polynomially decaying coefficients) to illustrate the gain in computational complexity of OOMP with respect to OMP.
	\item OMP was shown to require less data than Lasso for \textit{support recovery} (\cite{zhang2009consistency}). We consider the streaming sparse regression algorithm (SSR) presented in \cite{steinhardt2014statistics}, which is conceptually related to Lasso, as a benchmark to compare OOMP with $l_1$-regularization type algorithms. We prove that when $d>(s^*)^3$, OOMP outperforms SSR in terms of computational complexity.
	%\item To compare the runtime required for support recovery using OOMP ($C^{\text{OOMP}}$), with the corresponding runtime for the batch version $C^{\text{OMP}}$, we introduce the following notation: for a vector $\beta \in \mathbb{R}^d$ such that $\left|\text{support}(\beta)\right| = s$, we denote $\{ \beta_{(1)}, \dots, \beta_{(s)}\}$ its non-zero components ordered by magnitude (i.e: $\left|\beta_{(1)}\right|\ge \dots \ge \left|\beta_{(s)}\right|\ge \left|\beta_{(s+1)}\right|=0$ ) and introduce the quantity: $\norm{\beta}_{(k)}^2 := \frac{1}{k} \sum_{i=s-k}^s \beta_{(i)}^2$ (i.e the mean of the square of the smallest non-zero $k$ components of the vector $\beta$). We show that when $d > (s^*)^3$, neglecting logarithmic terms in the problem's parameters, we have:
%	\begin{equation*}
%	\frac{C^{\text{OOMP}}}{C^{\text{OMP}}} = \mathcal{O} \left(\frac{1}{s^*} \sum_{i=1}^{s^*} \frac{\norm{\beta^*}^2_{(1)}}{\norm{\beta^*}^2_{(i)}}\right).
%	\end{equation*}
%	Recall that $\norm{\beta^*}_{(1)}^2 = \beta_{(s^*)}^2$ hence the ratio is guaranteed to be at most $\mathcal{O}(1)$ in the most unfavourable case (where all the coefficients are equal). 
\end{itemize}

\paragraph{Organization} In section~\ref{sec:OMP}, we present high level ideas and key properties which underpin greedy feature selection principles such as the Orthogonal Matching Pursuit algorithm (in the batch as well as in the online setting). We then extend this idea and design a general Online OMP procedure which is built using two black-box procedures (namely \optimproc and \tryselectproc) in Section~\ref{sec:OOMP}. Then, we instantiate this general procedure using Algorithms~\ref{alg:optim} for \optimproc and~\ref{algo:select_mc} for \tryselectproc in Section~\ref{sec:instantiate}. Finally, we state theoretical guarantees about the output of the presented algorithm and provide a control on its runtime complexity. The last section presents simulations using synthetic data.

\subsection{Notations used}

Throughout the paper, we use the notation $[n]=\{1,\ldots,n\}$. We denote by $d$ the total input space dimension (total number of features),
and $s^{*}$ denotes the cardinality of the set $S^{*}$ of features to be recovered. For a vector
$\gamma \in \mbr^d$ and $F \subseteq [d]$, we denote $\gamma_{i:F}$ the coordinate of $\gamma$ corresponding to the $i$-th
element of $F$ ranked in increasing order, and $\gamma_F$ the vector of $\mbr^{|F|}$ such that $(\gamma_F)_i := \gamma_{i:F}$.
Similarly, for a matrix $M \in \mathbb{R}^{d \times d}$  we denote $M_F$ the matrix in $\mathbb{R}^{|F|\times |F|}$ obtained by restricting the matrix $M$ to the lines and columns with indices in $F$. For a random vector $x \in \mathbb{R}^{d}$, a random variable $y\in \mathbb{R}$ and $F \subseteq [d]$ we denote $\text{Cov}(x_F,y)$ the vector in $\mathbb{R}^{|F|}$ defined by $\text{Cov}(x_F,y)_i = \text{Cov}(x_{i:F}, y), \forall i \in [|F|]$. We denote $\Sigma$ the covariance matrix of $x$. For $\beta \in \mathbb{R}^{d}$ let us denote $\mathcal{R}(\beta) = \mathbb{E}_{(x,y)}[(y-\langle x, \beta \rangle)^{2}]$ the (population) squared risk function.

The prefix S refers to results presented in the supplementary material.

\section{Batch OMP and oracle version}\label{sec:OMP}

We start with recalling the standard batch OMP (Algorithm~\ref{algo:omp}) for reference.
Then we will introduce an ``oracle'' version when the data is random, which will serve as a guide for constructing the
online algorithm.

\subsection{Batch OMP}
Given a batch measurement matrix $\bm{X} \in \mathbb{R}^{n \times d}$ and a response vector $\bm{Y} \in \mathbb{R}^{n}$, at each iteration, OMP picks a variable that has the highest empirical correlation (in absolute value) with the ordinary linear least squares regression residue of the response variable with respect to features selected in the previous iterations. The algorithm stops when the maximum correlation is below a given threshold $\eta$.

\begin{algorithm}[tb] 
	%\centering
	\caption{OMP($\bm{X}$,$\bm{Y}$,$\eta$) \label{algo:omp}}
	\begin{algorithmic}
		\STATE $S = \emptyset$, $\bar{\beta} = 0$
		\WHILE {true}
		\STATE $\hat{i} \gets \text{argmax}_{j \notin S}|\bm{X}_{.j}^{t}(\bm{Y}-\bm{X}\bar{\beta})|$.
		\IF {$|\bm{X}_{.\hat{i}}^{t}(\bm{Y}-\bm{X}\bar{\beta})| < \eta$}
		\STATE \bfseries Break
		\ELSE
		\STATE $S \gets S \cup \{\hat i\}$
		\STATE $\bar{\beta} \gets \underset{\text{supp}(\beta) \subseteq S}{\text{argmin}} \|\bm{X}\beta-\bm{Y}\|^{2}$
		\ENDIF
		\ENDWHILE
		\STATE {\bfseries return}: $S$, $\bar{\beta}$.
	\end{algorithmic}
\end{algorithm}
\hfill
\begin{algorithm}[tb]
	%\centering
	\caption{Oracle OMP \label{algo:oracle_omp}}
	\begin{algorithmic}
		\STATE \textbf{Input}: integer $s^* (\infty \text{ if unknown}),\mu \in [0,1)$.
		\STATE Let $S=\emptyset$.
		\WHILE{$|S|<s^{*}$}
		\STATE Let $\beta^S = \underset{\text{supp}(\beta) \subseteq S}{\text{argmin}} \mathbb{E}_{(x,y)}[(y-\langle x, \beta \rangle)^{2}]$
		\STATE Let
		$Z_i^{S}  = \mathbb{E}[x_i(y-\langle x, \beta^S \rangle)], (i=1,\ldots,d).$
		\STATE Select $i^*$ such that:
		\STATE \hspace{5mm} $Z_{i^*}^{S} \in [\mu \max_{j\in [d]\setminus S}Z_j^{S}, \max_{j\in [d]\setminus S}Z_j^{S}]$
		\STATE \bfseries if  {$Z_{i^*}^{S} = 0$} then Break 
		\STATE $S \gets S \cup \{i^*\}$
		\ENDWHILE
		\STATE Output $S$.
		\STATE {\bf On interrupt:} {\bf return} $S$.
		
	\end{algorithmic}
\end{algorithm}

%The main argument of the existing theoretical analysis of OMP is that either the maximum correlation between candidate variables and the residue is small, typically smaller than a threshold $\eta$, then the algorithm stops and returns selected variables. Otherwise, this quantity is large, in this case, the corresponding variable is in $S^{*}$ with high probability.

Each iteration of Algorithm 1 comprises a selection procedure, where one selects a feature based on its correlation with the current residuals, and an optimization procedure, in this case the ordinary least squares, where one optimizes the squared loss function over the space spanned by the set of selected features, and determines the new residuals for the next iteration.

\subsection{Oracle OMP}

To understand why OMP works, we consider the setting where the data is random and present an ``oracle'' (or population) version of OMP in order to give an insight about the core principle of its selection strategy, which we will adapt to the streaming setting. Throughout this work we assume the following on
the generating distribution of feature vector and noise:
\begin{assumption}
	\label{ass:ass4}
	$\mathbb{E}[x] =0$, $y= \inner{\beta^*,x} + \epsilon$, and the noise variable satisfies $\mathbb{E}[\epsilon|x] =0$.	
\end{assumption}

Let us introduce the following classical assumption in support recovery literature, which appears in \cite{tropp2004greed, zhao2006model} and \cite{zhang2009consistency} as the irrepresentable condition (IC). Consider a subset $S \subseteq [d]$ and denote % by $\mu_S$:
\begin{equation*}
\mu_S = \max_{j \in [d] \setminus S} \|\Sigma_{S}^{-1} \text{Cov}(x_{S}, x_j)\|_1.
\end{equation*}

\begin{assumption}[Irrepresentable condition, IC]
	\label{ass:ass3}
	For all $S \subseteq [d]$ such that $|S| = s^{*}$,
	\begin{equation*}
	0\le \mu_S <1.
	\end{equation*}
\end{assumption}

\textbf{Remark:} The assumption $\mu_{S^{*}} < 1$ is often used for exact support recovery, it was shown in \cite{zhang2009consistency} that it is a {\em necessary} condition for the consistency of batch OMP feature selection. 

Consider for a subset $S \subseteq S^{*}$:
\begin{equation*}
\beta^S \in \underset{\text{supp}(\beta) \subseteq S}{\text{argmin}}\mathcal{R}(\beta).
\end{equation*}
We define the covariance between the oracle residuals with each feature as:
\begin{equation}\label{eq:defz}
Z_i^{S}  := \mathbb{E}[x_i(y-\langle x, \beta^S \rangle)], i=1,\ldots,d.
\end{equation}
%As for batch OMP, our
The selection criterion used in oracle OMP relies on the quantities $Z_i^S$, thanks to the following lemma:
\begin{lemma} 
	\label{lem:ineqzi}
	Suppose Assumptions~\ref{ass:ass4} and~\ref{ass:ass3} hold. For any $S \subseteq S^{*}$, we have
	(with the convention $\max \emptyset = 0$):
	\begin{equation}\label{ineq:fond}
	\underset{j \notin S^{*}}{\max} |Z_{j}^{S}| \le \mu_{S^{*}} \underset{i \in S^{*} \setminus S}{\max} |Z_{i}^{S}|. 
	\end{equation}
\end{lemma}
Algorithm~\ref{algo:oracle_omp} presents the resulting procedure, called Oracle version of OMP. In order to ease notations will use $\mu$ instead of $\mu_{S^{*}}$ in the remainder of this paper. 

\textbf{Remarks:} \\
$\bullet$ A similar result was used in~\cite{zhang2009consistency}
for the case of fixed design with random noise, where it was shown that either the empirical counterparts of $Z_i^S$ are small, or they satisfy an inequality analogous to~\eqref{ineq:fond}.  \\
$\bullet$ The right-hand side of~\eqref{ineq:fond} can be written as $\max_{i \in S^{*}} |Z_i^{S}|$, since $Z_i^{S} = 0$ for all $i \in S$.\\
$\bullet$ This lemma shows in particular that under Assumptions~\ref{ass:ass4}-\ref{ass:ass3}, if $S\subseteq S^*$ and $ \max_{i}|Z_i^{S}| >0$, then $\max_{i \notin S^{*}}|Z_i^{S}| < \max_{i \in S^{*}}|Z_i^{S}|$. Hence, unless $S^{*}= S$, picking the feature with the largest population correlation $|Z_i^{S}|$ guarantees that this feature belongs to $S^{*}$.\\
$\bullet$ In the oracle setting, the algorithm stops as soon as $\max_{i}|Z_i^{S}| =0$, since Lemma~\ref{lem:ineqzi}
guarantees that $S=S^*$ then. In the batch setting with a finite amount $n$ of available data, the algorihm stops
when the maximum empirical correlation is too small and and cannot guarantee $\max_{i}|Z_i^{S}| >0$ due to
estimation error. The threshold for stopping then depends on estimation error, hence on $n$, see~\cite{zhang2009consistency}.

\section{Online OMP}\label{sec:OOMP}

\subsection{Settings}

In a computation-resources-constrained setting, one aims at using the least possible queries of data points and features in order to gain in computational and memory efficiency. For a data point $(x,y) \in \mathbb{R}^{d}\times \mathbb{R}$, define $z \in \mathbb{R}^{d+1}$ by: $z_{[d]} = x$ and $z_{d+1} = y$.

In this paper, we focus on the the streaming data setting were one-pass over data is performed, as summarized above: 

The algorithm queries quantities through: \textbf{query-new}$(F)$, which takes as input $F \subseteq [d+1]$ and outputs the partial observation $z_{F}$ of a fresh data point  independent from all previously queried quantities. One call to \textbf{query-new}$(F)$ has a time complexity of $\mathcal{O}(|F|)$.

In what follows, we will split algorithms into subroutines and assume that the input of each subroutine only depends on the result of past queries. This ensures that all the new data accessed by a subroutine can be
considered as i.i.d. conditionally to its input.
More formally, let us denote by $\mathcal{F}_n$ the $\sigma$-algebra generated by all queried quantities up to the $n^{th}$ \textbf{query-new} query, and let $N$ be the (possibly random) number of queries made before the call to the current subroutine. Mathematically, $N$ is a stopping time; and, conditional to $\mathcal{F}_N$ the $K$ next calls to \textbf{query-new} produce an
i.i.d. sequence of (possibly partially observed) data points. We always assume that the input to each subroutine is $\mathcal{F}_N$-measurable.
Below we will analyse
each subroutine for a fixed input and derive probabilities with respect to the queried (i.i.d.) data;
in the global flow of the algorithm, under the above assumption the same probabilistic bounds will hold conditional to
$\mathcal{F}_N$.

\subsection{Algorithm}

Online OMP (Algorithm~\ref{algo:oomp}) selects variables sequentially.
%%% Gilles: retiré car me semble redondant/misleading
%Given a set of previously selected features, if the support $S^*$ is not completely recovered, OOMP selects (at least) an additional one through a call to \selectproc. 
In its general form, Algorithm \ref{algo:step} (\selectproc)~consists of two sub-routines: \optimproc and \tryselectproc. The first provides an approximation of the regression coefficients for features in $S$. The latter is an approximate best arm identification strategy which uses the output of \optimproc and queries data points in order to try to select feature $i$, such that $Z_i^S$ is large enough  (Lemma~\ref{lem:ineqzi} shows that such a feature is in $S^{*}$). We now describe how \optimproc and \tryselectproc operate:

\begin{algorithm}[tb]
	\caption{Online OMP($\delta, s^*$)\label{algo:oomp}}
	\begin{algorithmic}
		\STATE \textbf{Input}: $s^* (\infty \text{ if unknown})$, $\delta \in (0,1)$
		\STATE \textbf{Input}: $\mu \in(0,1),\rho > 0$ (globals)
		\STATE Let $S=\emptyset$.
		\WHILE{$|S| < s^*$} 
		\STATE $U \gets \selectproc(S, \frac{\delta}{2(|S|+1)(|S|+2)}, 1)$
		\STATE $S \gets S \cup U$
		\ENDWHILE
		\STATE Return: $S$
		\STATE {\bf On interrupt:} {\bf return} $S$
	\end{algorithmic}
\end{algorithm}
\hfill
\begin{algorithm}[tb]
	\caption{$\selectproc$($S$,$\delta$,$\xi$) \label{algo:step}}
	\begin{algorithmic}
		\STATE [Globals: $\mu \in (0,1),\rho \in (0,1)$]
		\STATE $\tilde{\beta} \gets \optimproc(S,\delta,\xi)$
		\STATE $(U, \text{Success}) \gets \tryselectproc(S,\delta, 
		\tilde{\beta},\xi)$
		\IF {$\neg$Success}
		\STATE Return: {\bfseries Select}($S, \delta/2,\xi/4$)
		\ELSE
		\STATE {\bfseries return} $U$
		\ENDIF
	\end{algorithmic}
\end{algorithm}

\paragraph{\optimproc sub-routine:}is assumed to be a black-box optimization procedure such that for any fixed subset $S \subseteq [d]$, positive number $\xi$ and $\delta \in (0,1)$, $\optimproc(S, \delta, \xi)$ queries fresh data points through \textbf{query-new}$(S\cup \{d+1\})$ and outputs an approximation $\tilde{\beta}^S$ for $\beta^S$. We say that \optimproc  satisfies the \textit{optimization confidence property}
if
\begin{equation}
\label{eq:condopt}
\prob{ \mathcal{R}\left(\tilde{\beta}^S \right)- \mathcal{R}\left( \beta^S\right) > \xi \,\Big|\, S,\delta,\xi} \leq \delta,
\end{equation}

where the probability is with respect respect to the data queried during the procedure, for any fixed input $(S, \delta, \xi)$.

\paragraph{\tryselectproc sub-routine:} Given a set of selected features $S$, an (approximate) regression coefficients vector $\tilde{\beta}^S$ and a confidence bound $\xi$ (on $\tilde{\beta}^S$), $\tryselectproc(S, \delta, \tilde{\beta}^S, \xi)$ queries fresh data points to approximate $Z_i^S$ defined by~\eqref{eq:defz} for $i \in [d]\setminus S^{*}$ and either returns {\tt Success=False}, or {\tt Success=True} along with a set $U$ of new selected features.

We say that \tryselectproc satisfies the \textit{selection property} if for any (fixed) input $(S, \delta, \tilde{\beta}^S, \xi)$, it holds for the (random)  output $({\tt Success}, U)$:
\begin{multline}
\label{eq:selcond}
\text{provided } S\subseteq S^* \text{ and } \mathcal{R}\big(\tilde{\beta}^S\big) - \mathcal{R}\left(\beta^S\right) \leq\xi, \text{it holds: }\\
\prob[2]{\overline{A}({\tt Success},U) \,\big|\, S, \delta, \tilde{\beta}^S, \xi } \leq \delta, \\
\text{ where }  \overline{A}({\tt Success},U) :=\big\{{\tt Success = True}; 
  \exists i \in U: \mu_{S^{*}} \max_{j \in S^*\setminus S} \abs[1]{ Z^S_j} \geq \abs[1]{Z^S_{i}}\big\},
\end{multline}
where the probability is with respect to all data queries made by \tryselectproc for fixed input.
This implies in particular that $U \subset S^*\setminus S$ with probability $1-\delta$, by Lemma~\ref{lem:ineqzi} (and in particular, with the convention $\max \emptyset = 0$, the
probability of returning $\tt Success=True$ when $S=S^*$ is less than $\delta$).

If \tryselectproc returns $\tt Success = False $, this suggests
that the bound $\xi$ is not tight enough, i.e. that the prescribed precision $\xi$ for the optimization part is insufficient to find a feature with the
guarantee~\eqref{eq:selcond} holding with the required probability.
In this case, 
using the doubling trick principle, $\selectproc$ is called 
recursively with the input $\left(S, \delta/2, \xi/4 \right)$.
Algorithm~\ref{algo:step} presents the general form of the procedure \selectproc.

If the cardinality $|S^*|=s^*$ is not known in advance, there is no stopping criterion
and the procedure is run indefinitely. We assume that Online OMP
will be interrupted externally by the user based on some arbitrary criterion, for
example a limit on total computation time or other resource. In this case the current
set $S$ of selected features is returned. The next lemma ensures that at any interruption
time, it is guaranteed with high probability that $S\subseteq S^*$.

\begin{lemma}\label{global-lem}
	Suppose that Assumptions~\ref{ass:ass3} and~\ref{ass:ass4} hold. Consider Algorithm~\ref{algo:oomp} with the procedure \selectproc given in Algorithm~\ref{algo:step}, assume that \optimproc satisfies the {\em optimization confidence property}~\eqref{eq:condopt} and that \tryselectproc satisfies the \textit{selection property}~\eqref{eq:selcond}. Then when  \textbf{OOMP}($\delta, s^*$) (Algorithm~\ref{algo:oomp}) is terminated, the variable $S$ satisfies with probability at least $1-2\delta$: $S \subseteq S^*$.
\end{lemma}  

\textbf{Remark:} The above result only guarantees that the recovered features
belong to the true support. We will see later in Lemma~\ref{lem:tau} that for the instantiations of \tryselectproc and \optimproc considered in the next section, unless the support $S^*$ is completely recovered, the procedure \selectproc finishes in finite time. Together with the previous lemma, this guarantees that the support $S^*$ will be recovered in finite time with high probability, at which point \selectproc will enter an infinite loop of recursive calls until interruption. In Section~\ref{se:complexitytheory}, we will derive quantitative bounds on
the complexity for recovering the full support.

\textbf{About the stopping rule:} %The information-theoretic considerations imply that it is impossible (whatever the technique used) to obtain, without any assumption on the signal coefficients, a high confidence upper bound on s from finite data (or even decide between $s=0$ vs $s>0$ i.e. the signal detection problem), since the signal can be arbitrarily faint.
%In the batch setting, OMP uses a fixed data set, the stopping rule introduced in \cite{zhang2009consistency} depends on the size n of the available data set used and the algorithm can only recover features whose coefficient is larger than certain detection threshold depending on n. It stops when it is certain (with high probability) that the remaining coefficients must be below that threshold. 
OOMP has access to a virtually infinite stream of data points, so unless it is halted externally by the user, the algorithm can (in principle) continue querying more data to search for potentially extremely small coefficients (in contrast to
the batch setting where the amount of available data is limited). However it is possible, in every call of the procedure Try-Select, to communicate to the user an upper bound on the maximal magnitude of the remaining
coefficients of variables in $S^*\setminus S$ %that the algorithm is trying to recover
(as shown in Section~\ref{sec:algos}). Therefore, the user can halt the procedure whenever that bound is small enough (alternatively, a threshold can be passed as an input to the algorithm and a corresponding stopping rule can be derived). We advocate an agnostic point of view where the user can decide for themselves when to halt the algorithm (based on the information on the magnitude of the remaining coefficients, but also possibly on limitations of the size of available data or computation time). Our recovery result guarantees that stopping at any time, the set of selected variables is (with high probability) a subset of $S^*$.

\section{Instantiation of the Optimization procedure and Selection Strategy}\label{sec:instantiate}
In this section we provide an instantiation of \tryselectproc and \optimproc procedures.

\subsection{Assumptions}
In addition to the Irrepresentable Condition (IC) (Assumption~\ref{ass:ass3} ) we will make an assumption of 
Restricted Isometry Property (RIP) \cite{tropp2004greed,zhang2009consistency,wainwright2009sharp} for the distribution of $(x,y)$.
Denote $\Lambda^{\min}_S$ and $\Lambda^{\max}_S$ the lowest and largest eigenvalue of $\Sigma_S$ respectively.
\begin{assumption} 
	\label{ass:ass1}
	[RIP] %Suppose that there exists $L, \rho >0$ such that
	For all $S \subseteq [d]$ such that $|S| = s^{*}$, it holds
	$	0<\rho\le \Lambda^{\min}_S, \Lambda^{\max}_S \le L.$	
\end{assumption} 
We also make the following assumption:
\begin{assumption}
	\label{ass:ass2}
	Assume that  $ |y|<1$ and $\|x\|_{\infty} < M$ (a.s.).
	
\end{assumption}

\subsection{Instantiation of \optimproc and \tryselectproc}

Recall that one call of the procedure \selectproc results in successive calls of \optimproc and \tryselectproc until (at least) a feature is selected. Moreover, the quantities queried in a sub-routine call (either \tryselectproc or \optimproc) are independent from quantities queried during the execution of previous functions.

\paragraph{Optimization procedure:}

We opted for the averaged stochastic gradient descent (Algorithm~\ref{algo:optimsagd}). High probability bounds on the output of this procedure were given in \cite{harvey2019simple}. We use this finding to build an optimization procedure satisfying the \textit{optimization confidence property}~\eqref{eq:condopt} for an input $(S, \delta, \xi)$.

\begin{proposition}\label{prop:optim_prop}
	Let Assumptions~\ref{ass:ass4},\ref{ass:ass3}, \ref{ass:ass1} and~\ref{ass:ass2} hold. Then Algorithm~\ref{alg:optim} satisfies the \textit{optimization confidence property}.
\end{proposition}

\begin{algorithm}[tb]
	\caption{\optimproc($S$, $\delta$, $\xi$)\label{alg:optim} \label{algo:optimsagd}}
	\begin{algorithmic}
		\STATE {\bfseries Input:} initial  $\beta_0$, $\delta$, $\xi$
		\STATE Let $\tilde{\beta}_0 = \beta_0$, $\mathcal{X} = \mathcal{B}_{|S|}(0,\frac{2}{\sqrt{\rho}})$
		%		\STATE Query $m$ new samples $(X_t,Y_t)$, $t=1,\ldots,m$
		\STATE $G \gets 10 \left|S\right| \frac{M^2}{\sqrt{\rho}} + 2 \sqrt{\left|S\right|} M$
		\STATE Let $T \gets 21G^2 \log\left(1/\delta\right)/\left( \rho \xi\right)$
		\FOR{$t\gets 0,...,T-1$}
		\STATE $\eta_t \gets \frac{2}{\rho (t+1)}$, $\nu_t \gets \frac{2}{t+1}$
		\STATE $(X,Y) \gets \textbf{query-new}(S\cup \{d+1\})$
		\STATE $\gamma_{t+1} \gets \beta_{t} - 2\eta_t (X^{t}\beta_{t}-Y)X$
		\STATE $\beta_{t+1} \gets \Pi_{\mathcal{X}}(\gamma_{t+1})$ 
		\STATE //\text{where $\Pi_{\mathcal{X}}$ is the projection operator on $\mathcal{X}$}
		\STATE $\tilde{\beta}_{t+1} \gets (1-\nu_{t})\tilde{\beta}_{t}+\nu_{t} \beta_{t+1}$
		\ENDFOR
		\STATE $\textbf{return   } \tilde{\beta}_T$
	\end{algorithmic}
\end{algorithm}

\paragraph{\tryselectproc Strategy:} 

Different approximate best arm identification strategies were developed in the literature. In this work, we opt for a LUCB-type strategy were we use some ideas from \cite{mason2020finding}. 
We approximate $Z_i^{S}$ by (i) replacing $\beta^S$ by an approximation $\tilde{\beta}^S$ assumed to satisfy the condition
$\mathcal{R}\left(\tilde{\beta}^S\right) - \mathcal{R}\left( \beta^{S}\right) \leq \xi $; (ii) replacing the expectation by an empirical counterpart using
queried quantities. Given an i.i.d sequence {\bf $(X_h,Y_h), h \ge 1$}, we define   $\tilde{Z}_{i,n}^{S}(\tilde{\beta}^S)$ and $\tilde{V}_{i,n}(\tilde{\beta}^S)$ for $n \ge 2$, using {\bf $(X_h,Y_h), 1 \leq h \leq n$} written in matrix and vector
form as $\bm{X} \in \mathbb{R}^{n \times d}$, $\bm{Y} \in \mathbb{R}^n$ by:
\begin{align*}
\tilde{Z}_{i,n}^{S}(\tilde{\beta}^S) &:= \frac{1}{n} \bm{X}_{.i}^{t}(\bm{X}\tilde{\beta}^S-\bm{Y}),
i=1,\ldots,d;\\
\begin{split}
\tilde{V}_{i,n}(\tilde{\beta}^S) &:= \frac{1}{n(n-1)}\\ 
&\sum_{1\le h,l \le n} \left( \bm{X}_{i,h}(\bm{X}\tilde{\beta}^S -\bm{Y})_h - \bm{X}_{i,l}(\bm{X}\tilde{\beta}^S -\bm{Y})_l  \right)^{2};
\end{split}\\
\tilde{V}_{i,n}^+(\tilde{\beta}^S) &:= \max\left\lbrace \tilde{V}_{i,n}(\tilde{\beta}^S); \frac{1}{1000}\frac{LM^2}{\rho} \right \rbrace.
\end{align*}

Note that $\tilde{V}_{i,n}(\tilde{\beta}^S)^+$ represents a thresholded version of the empirical variance $\tilde{V}_{i,n}(\tilde{\beta}^S)$.
Proposition~\ref{prop:conczi} gives a concentration inequality for $\tilde{Z}_{i,n}^{S}$, using empirical Bernstein bounds \cite{maurer2009empirical}. 	%We define $\text{conf}(i,n,\delta)$
For $i \in [d]\setminus S, n\ge 2$ and $\delta \in (0,1)$, define $\tilde B(\tilde{\beta}^S) := M^2 \|\tilde{\beta}^S\|_1+ M$ and:
\begin{equation}
\label{eq:defconf}
\mathrm{conf}(i,n,\delta) := \sqrt{\frac{8\tilde{V}_{i,n}^+(\tilde{\beta}^S)\log(8dn^2/\delta)}{n}} 
+ \frac{28 \tilde B(\tilde{\beta}^S)\log(8dn^2/\delta)}{3(n-1)}.
\end{equation}

\begin{proposition}
	\label{prop:conczi}
	Consider a fixed subset $S \subseteq S^{*}$
	and put $k:=|S|$. Suppose Assumptions~\ref{ass:ass4}, \ref{ass:ass3}, \ref{ass:ass1} and \ref{ass:ass2} hold. 
	Assume to be given a fixed $\tilde{\beta}^S \in \mathbb{R}^d$ with support $S$, satisfying~$\mathcal{R}\left(\tilde{\beta}^S\right) - \mathcal{R}\left( \beta^{S}\right) \leq \xi  $. For all $\delta \in (0,1)$, with probability at least $1-\delta$ it holds:
	\begin{equation}\label{eq:conczi}
	\text{for all } i\in [d] \setminus S, \text{ and }  n\ge 2:
	\quad |\tilde{Z}_{i,n}^S(\tilde{\beta}^S)-Z_i^S| \le
	\frac{1}{2} \mathrm{conf}(i,n,\delta) +M\sqrt{\xi}.
	\end{equation}
\end{proposition}

Proposition~\ref{prop:conczi} entails the following: conditionally to $S \subseteq S^{*}$, for all $\delta \in (0,1)$, with probability at least $1-\delta$: for all $ i \in [d]\setminus S, n \ge 2$, the condition
$2M \sqrt{\xi}  < \mathrm{conf}(i,n,\delta)$ implies
\begin{equation}\label{eq:prac}
|\tilde{Z}_{i,n}^{S}-Z_i^{S}| \le \mathrm{conf}(i,n,\delta).
\end{equation}
Provided inequality~\eqref{eq:prac} holds true, and let $\hat{i} \in \text{argmax} \{ | \tilde{Z}^S_{i,n}| + \mathrm{ conf}(i,n,\delta)\}$, then, if $j \in [d]\setminus S$ satisfies the following condition:
\begin{equation}\label{eq:cond}
|\tilde{Z}_{j,n}^S| - \mathrm{conf}(j,n,\delta) \ge \mu \left(|\tilde{Z}_{\hat{i},n}^S|+ \mathrm{conf}(\hat{i},n,\delta) \right) ,
\end{equation}
then it holds that $\abs[1]{Z_{j}^S} > \mu \max_{i \in S^*} \abs{Z_{i}^S}$ (see Lemma~\ref{lem:selprop} for a proof). Thus, in view of Proposition~\ref{prop:conczi}, under the above conditions, an algorithm selecting features $j$ satisfying~\eqref{eq:cond} satisfies the {\em selection property}.
% The search for $i^*$ satisfying the above conditions differs according to the setting:

Using this observation, we build Algorithm~\ref{algo:select_mc} as follows:
the procedure repeatedly queries fresh data points $(x,y)$ and updates the quantities $\tilde{Z}_{i,n}^S$ simultaneously for all $i\in [d] \setminus S$. After each iteration, we pick $\hat{i} \in \text{argmax} \{ | \tilde{Z}^S_{i,n}| + \text{ conf}(i,n,\delta)\}$ and we eliminate features
for $j$ which we are certain that $j \not\in \text{argmax}_i | Z^S_{i}|$ (i.e suboptimal features) with high probability through the test:
\begin{equation*}
\left| \tilde{Z}^S_{j,n}\right|  + \text{ conf}(j,n,\delta) < \left| \tilde{Z}^S_{\hat{i},n}\right|  - \text{ conf}(\hat{i},n,\delta).
\end{equation*}
Moreover, we select features satisfying the condition~\eqref{eq:cond}. The procedure halts when the condition:
\begin{equation*}
\left| \tilde{Z}^S_{\hat{i},n}\right| \le \frac{2}{1- \mu} \text{ conf}(\hat{i},n,\delta)
\end{equation*}
is no longer satisfied. The algorithm then returns the set of selected features $U$. Lemma~\ref{lem:tau} shows that unless the support $S^*$ is completely recovered, $U \neq \emptyset$ and the procedure halts in finite time almost surely. 
A concise version of \tryselectproc is given in Algorithm~\ref{algo:select_mc} (the detailed version is in Algorithm~\ref{algo:select_mc_full}).

%The procedure repeatedly queries fresh data points $(x,y)$ and updates the quantities $\tilde{Z}_{i,n}^S$ simultaneously for all $i\in [d] \setminus S$, until one of the candidate features $i^* \in \text{argmax} \{ | \tilde{Z}^S_{i,n}| + \text{ conf}(i,n,\delta)\}$ satisfies~\eqref{eq:cond}.
%Then the procedure halts and outputs the feature $i^*$. 
%A concise version of \tryselectproc in the Data Stream setting is given in Algorithm~\ref{algo:select_mc} (the detailed version is
%in the appendix in Algorithm~\ref{algo:select_mc_full}).

\begin{algorithm}[H]
	\caption{\tryselectproc($S$, $\delta$, $\tilde{\beta}$, $\xi$) 
 \label{algo:select_mc}}
	\begin{algorithmic}
		\STATE {\bfseries Input:} $S$, $\delta$, $\tilde{\beta}$, $\xi$ \quad \COMMENT{$\tilde{\beta}$ is of dim. $|S|$}
		
		\STATE {\bfseries Output:} $S$, $\text{Success}$
		\STATE Let $v,Z,\text{conf}$ be $d$-arrays \\
		\hspace{2mm} \COMMENT{will store $\tilde{V}_{i,n},\tilde{Z}_{i,n}^S$ and $\text{conf}(i,n)$}
		\STATE $n \gets 0$, $Z \gets \bm{0}$, $v \gets \bm{0}, U \gets \emptyset, L \gets [d+1] \setminus S$
		\WHILE{True}
		\STATE $n \gets n+1$
		\STATE $(X,Y) \gets \textbf{query-new}(L)$
		\FORALL{$i \in \{1,\ldots d\}$}
		\STATE $Z[i] \gets \frac{1}{n} X_i (Y-X_S^{t}\tilde{\beta}) + \frac{n-1}{n} Z[i]$
		\STATE Update $v[i]$
		\STATE $\text{conf}[i] \gets \text{conf}(i,n)$
		\ENDFOR
		\IF {$2M \sqrt{\xi} > \min_i \text{conf}[i]$}
		\STATE $\text{Success} \gets \text{False}$, \textbf{break}
		\ENDIF
		\STATE $\hat{i} \gets \underset{i \in [d]\setminus S}{\text{argmax}} \{|Z[i]|+\text{conf}[i]\} $
		\FORALL{$i \in L \setminus \{d+1\}$}
		\IF {$|Z[i]| + \text{ conf}[i]\le   \left|Z[\hat{i}] \right| -   \text{ conf}[\hat{i}]$}
		\STATE $L \gets L \setminus \{i\}$
		\ENDIF
		\IF {$|Z[i]|- \text{ conf}[i] \ge  \mu \left( \left| Z[\hat{i}]\right| + \text{ conf}[\hat{i}]\right) $}
		\STATE $U \gets U \cup \{i\}$
		\ENDIF
		\ENDFOR
		\IF {$|Z[\hat{i}]|> \frac{2}{1-\mu} \text{ conf}[\hat{i}]$}
		%\STATE $S \gets S \cup U$
		\STATE $\text{Success} \gets \text{True}$, \textbf{break}
		\ENDIF
		\ENDWHILE
		\STATE $\textbf{return  } U, \text{Success}$
	\end{algorithmic}
\end{algorithm}
\hfill

\section{Theoretical Guarantees and Computational Complexity Analysis}

\label{se:complexitytheory}

Consider one call of $\selectproc(S, \delta, 1)$, for a fixed $S \subseteq S^{*}$. Lemma~\ref{lem:tau} below shows that, unless the support of $S^{*}$ is totally recovered, the procedure $\selectproc(S, \delta, 1)$ halts in finite time and updates $S$ with a non-empty set of features.

\begin{lemma}\label{lem:tau}
	Suppose Assumptions~\ref{ass:ass4},\ref{ass:ass3},\ref{ass:ass1} and~\ref{ass:ass2} hold. Consider one call of  $\selectproc(S,\delta,1)$ where \tryselectproc is given by Algorithm~\ref{algo:select_mc}, and \optimproc is given by Algorithm~\ref{alg:optim}. Denote by $\tau$ the stopping time where $\selectproc(S,\delta,1)$ updates $S$ with the set of selected features $U$   (i.e the subroutine \tryselectproc returns $U$ and $\tt Success = True$), then :
	
	If $S\subsetneq S^{*}$: $\mathbb{P}(\tau <+\infty \text{ and } U \neq \emptyset) = 1$.
	
	If $S = S^{*}$: $\mathbb{P}(\tau = +\infty) \ge 1-2\delta$.
\end{lemma}

%\begin{lemma}\label{lem:tau}
%	Suppose Assumptions~\ref{ass:ass4},\ref{ass:ass3},\ref{ass:ass1} and~\ref{ass:ass2} hold. Consider one call of  $\selectproc(S,\delta,1)$ where \tryselectproc is given by Algorithm~\ref{algo:select_mc}, and \optimproc is given by Algorithm~\ref{alg:optim}. Denote by $\tau$ the stopping time where $\selectproc(S,\delta,1)$ selects a feature (i.e the subroutine \tryselectproc returns $\tt Success = True$), then :
%	
%	If $S\subsetneq S^{*}$: $\mathbb{P}(\tau <+\infty) = 1$.
%	
%	If $S = S^{*}$: $\mathbb{P}(\tau = +\infty) \ge 1-2\delta$.
%\end{lemma}

Let  $S \subsetneq S^{*}$ be a fixed subset and denote $k:=|S|$. Recall that running $\selectproc(S,\delta,1)$ results in executing \optimproc~and \tryselectproc~alternatively (see Algorithm~\ref{algo:step}). Let us denote by $C_{\optimproc}^S$ the cumulative computational complexity of \optimproc when running $\selectproc(S,\delta,1)$ and by $C_{\tryselectproc}^S$ the cumulative computational complexity of \tryselectproc when running $\selectproc(S,\delta,1)$.

\begin{theorem}\label{th:main}
	Suppose Assumptions~\ref{ass:ass4}, \ref{ass:ass3}, \ref{ass:ass1} and ~\ref{ass:ass2} hold. Consider the procedure \selectproc given by Algorithm~\ref{algo:step}, \tryselectproc given by Algorithm~\ref{algo:select_mc}, and \optimproc as in Algorithm~\ref{alg:optim}. Assume that $S \subsetneq S^*$ and denote $k:=|S|$. Then $\selectproc(S, \delta, 1)$ selects a non-empty set of additional features $ U$ such that: 
	\begin{equation*}
	\mathbb{P} \left( U \subset S^*\right) \ge 1-2\delta.
	\end{equation*} 
	Moreover, the computational complexity of $\selectproc(S, \delta, 1)$ subroutines \optimproc and \tryselectproc satisfy  with probability at least $1-\delta$:	
	\begin{align*}
	C_{\optimproc}^S &\le \kappa k^3 \max\left\lbrace \frac{1}{W_{i^*}^2} , \frac{\sqrt{k}}{W_{i^*}} \right\rbrace \log \left( \frac{\bar{k}}{ \delta W_{i^*}}\right); \\
	C_{\tryselectproc}^S  & \le \kappa \sum_{i \in [d]\setminus S} \max\left\lbrace \frac{1}{W_{i}^2} ; \frac{\sqrt{\bar{k}}}{W_{i}} \right \rbrace
	 \log \left( \frac{d}{ \delta W_{i^*}}\right)\log \left( \frac{\bar{k}}{ W_{i^*}}\right);\\
	\end{align*}
	where $i^* \in \mathop{\mathrm{argmax}}_{i \in S^*\setminus S} \abs{Z^S_i}$;
	$W_i := \max((1-\mu)\abs{Z_i^S},\abs{Z_{i^*}^S} - \abs{Z_i^S})$;
	$\bar{k} = \max\left\lbrace 1, k\right \rbrace$ and $\kappa$ is a constant depending only on $\rho, L$ and $M$.	
\end{theorem}%\todo{Vérifier: $\bar{k}$, $\delta_k$}

Theorem~\ref{th:main} provides high probability bounds on the computational complexity for a call to the procedure \selectproc.
A crucial point is that the complexity of the  
$k$-th step depends on the largest correlation $\abs{Z_i^S}$ over the remaining (yet unselected)
features, which in turn can be related to the average of the corresponding coefficients of $\beta^*$
(see Lemma~\ref{lem:ord}). By contrast,
due to the batch nature of OMP, its complexity is driven by the minimum coefficient of $\beta^*$,
which determines the minimum amount of needed data for full recovery.

Let us introduce the following notation: let $\left(\beta_{(i)} \right)_{1 \le i \le s^*}$ be the coefficients of $\beta^*$ ordered in decreasing sequence of magnitude. Let $\tilde{\beta}_{(s^*-k+1)}^2$ denote the average of the square of the $k$ smallest non-zero coefficients of $\beta^*$: $\tilde{\beta}_{(s^*-k+1)}^2 := \frac{1}{k} \sum_{i=s^*-k+1}^{s^*} \beta_{(i)}^2$.

\begin{corollary}\label{corollary1}
	Under the same assumptions as theorem~\ref{th:main}. The computational complexity of $\selectproc(S, \delta, 1)$ subroutines \optimproc and \tryselectproc satisfy  with probability at least $1-\delta$:	
	\begin{align*}
	C_{\optimproc}^S &\le \kappa  \frac{k^3}{\tilde{\beta}^2_{(k+1)}}  \log\left(\frac{\bar{k}}{\delta \tilde{\beta}^2_{(k+1)}}\right); \\
	C_{\tryselectproc}^S  & \le \kappa  \frac{d}{\tilde{\beta}^2_{(k+1)}}  \log\left(\frac{\bar{k}}{ \tilde{\beta}^2_{(k+1)}}\right)\log\left(\frac{d}{\delta \tilde{\beta}^2_{(k+1)}}\right);
	\end{align*}
	where $\kappa$ is a constant depending only on $\rho, L, M, \mu$, and $\bar{k} = \max\{k,1\}$.
\end{corollary}%\todo{Il faut être plus précis sur les ``log terms''}
We use bounds of corollary~\ref{corollary1} to compare the computational complexity of OOMP with the computational complexity of OMP using the sample size prescribed by \cite{zhang2009consistency} for full support recovery. Then, we compare OOMP with the SSR algorithm presented in \cite{steinhardt2014statistics} for streaming sparse regression, as a Lasso-type procedure. We use Theorem 8.2 in \cite{steinhardt2014statistics} to derive a sufficient sample size to achieve full support recovery.

We denote by $C^{\text{OOMP}}$ the total runtime necessary for OOMP in order to recover the support completely, and denote by $C^{OMP}$ and $C^{SSR}$ the corresponding quantities for OMP and SSR respectively.

\begin{corollary}\label{corollary2}
	Under the same assumptions as theorem~\ref{th:main}. If $d>(s^*)^3$, we have with probability at least $1-\delta$:
	%	\begin{align*}
	%	\frac{C^{\text{OOMP}}}{C^{\text{OMP}}} &\le \frac{\kappa \log(d/\beta_{(s^*)}\delta)\log(s^*/\beta_{(s^*)}\delta)}{\log(d/\delta)} \left(1+\frac{(s^*)^3}{d} \right) \frac{1}{s^*}\sum_{i=1}^{s^*} \frac{\beta_{(s^*)}^2}{\tilde{\beta}_{(i)}^2}; \\
	%	\frac{C^{\text{OOMP}}}{C^{\text{SSR}}}  & \lesssim \frac{\kappa \log(d/\beta_{(s^*)}\delta)\log(s^*/\beta_{(s^*)}\delta)}{\log(d/\delta)} \left(1+\frac{(s^*)^3}{d}\right) \frac{1}{(s^*)^2}\sum_{i=1}^{s^*}  \frac{\beta_{(s^*)}^2}{\tilde{\beta}_{(i)}^2}  ;\\
	%	\end{align*}
	\begin{align*}
	\frac{C^{\text{OOMP}}}{C^{\text{OMP}}} &\le  \kappa\log^2\left(\frac{s^*}{\beta_{(s^*)}^2}\right) \frac{1}{s^*}\sum_{i=1}^{s^*} \frac{\beta_{(s^*)}^2}{\tilde{\beta}_{(i)}^2} ; \\
	\frac{C^{\text{OOMP}}}{C^{\text{SSR}}} &\le  \kappa \log^2\left(\frac{s^*}{\beta_{(s^*)}^2}\right) \frac{1}{(s^*)^2}\sum_{i=1}^{s^*} \frac{\beta_{(s^*)}^2}{\tilde{\beta}_{(i)}^2} ;
	\end{align*}
	where $\kappa$ is a constant depending only on $\rho, L, M$ and $\mu$.
\end{corollary}

Recall that we have $\forall i \in [s^*]: \beta^2_{(s^*)} \le \tilde{\beta}^2_{(i)}$. Hence: $\frac{1}{s^*}\sum_{i=1}^{s^*} \frac{\beta_{(s^*)}^2}{\tilde{\beta}_{(i)}^2} \le 1$, with equality only if all the square of the coefficients are equal. The SSR complexity bound have and additional factor $\frac{1}{s^*}$, the same factor appears when comparing the sample size used by OMP for support recovery $n^{\text{OMP}}$ in \cite{zhang2009consistency}, with the corresponding quantity for Lasso $n^{\text{Lasso}}$ in \cite{zhao2006model}: $n^{\text{OMP}} = \mathcal{O}(\frac{n^{\text{Lasso}}}{s^*})$. Since our objective is support recovery, we will focus on the comparison between OOMP and OMP in the remainder of this paper. 

In order to illustrate the advantage of OOMP over OMP,  we consider the specific
situation where the coefficients of $\beta^*$ decay polynomially as: $\beta_i = \frac{1}{\sqrt{s^*}} \left(1 - \frac{i-1}{s^*}\right)^{\gamma}$, for $i \in S^*$ and $\beta_i = 0$ for $i\notin S^*$; with $\gamma\geq 0$ and we assume that $d>(s^*)^3$. Then we have, with probability at least $1-\delta$:
\begin{equation}\label{eq:specific_sc}
\frac{C^{\text{OOMP}}}{C^{\text{OMP}}} \le \kappa \frac{\log^2\left(s^*\right)}{(s^*)^{\min\left\lbrace 2\gamma,1\right\rbrace}}.
\end{equation}
%If $\gamma = \frac{1}{2}$:\todo{Je ne vois pas la nécessité de distinguer ce cas?}
%
%\begin{equation*}
%\frac{C^{\text{OOMP}}}{C^{\text{OMP}}} \le \kappa \left( \frac{1}{s} + \frac{s^2 }{d}\right),
%\end{equation*}
where $\kappa$ is a constant depending only on $\rho, L, M$ and $\mu$.
See section~\ref{sec:comp_compare} for a proof of the results above. Thus,
in a typical scenario of coefficient decay ($\gamma>0$), OOMP reduces the complexity of OMP by
a large factor (observe that the worst case in this scenario is $\gamma=0$, i.e. when
all coefficients all are of the same order, which is not the typical case in practice).

\section{Simulations}\label{sec:sims}

In this section, we aim at comparing the computational complexities of OOMP and OMP. We denote $n^{\text{OMP}}$ the sample size prescribed by-
\cite{zhang2011sparse} (recalled as Theorem~\ref{th:omp})
to fully recover the support using OMP. We consider $C^{\text{OMP}} = s^*dn^{\text{OMP}} + (s^*)^2 n^{\text{OMP}}$ as a proxy for the computational complexity of OMP. For OOMP, we use Lemma~\ref{lem:upc} and evaluate $C^{\text{OOMP}}$ as a function of the quantity of data points queried.

From a practical point of view,  the number of iterations theoretically
prescribed in the optimization procedure (the number $T$ in Algorithm~\ref{alg:optim}),
and coming from \cite{harvey2019simple} is very pessimistic, due to the large numerical constant up to which the confidence bounds of the averaged stochastic gradient descent were developed.
Taking this theoretical prescription to the letter resulted in the \optimproc step demanding
an inordinate amount of data compared to \tryselectproc, while we expect the latter step to
carry the larger part of the complexity burden due to the influence of the dimension $d$.
For this reason, in our simulation we opted to significantly reduce this numerical constant,
while ascertaining  (since we know the ground truth) that the {\em optimization confidence property}~\eqref{eq:condopt} was still satisfied in practice in all simulations.

We generate samples $(x_t, y_t)$ with each coordinate of $x_t$ distributed as $\mathrm{Unif}\left[-B ; B\right]$ with $B=0.5$ and $y_t = \langle x_t, \beta^* \rangle + \epsilon_t$. We pick $\beta^*$ to be a sparse vector with $s^* = \log_2(d)$ non zero coordinates and $\epsilon_t \sim \mathrm{Unif} \left([-\eta, \eta]\right)$, where $\eta = 0.5$. We consider the case where the coefficients of $\beta^*$ decay linearly: $\beta^*_i = \frac{1}{\sqrt{s^*}} \left(1 - \frac{i-1}{s^*}\right)$ for $i\in [s^*]$ and $\beta_i^* = 0$ if $i>s^*$. We consider two scenarios for the structure of the correlation matrix $\Sigma$: the orthogonal design $\Sigma_{\mathrm{orth}} = I_d$ and the power decay
Toeplitz design, with parameter $\phi = 0.1$:
% \begin{equation*}
% \Sigma_{\text{orthogonal}} = \begin{pmatrix}
% 1      & \cdots & 0 \\ 
% \vdots & \ddots & \vdots \\ 
% 0      & \cdots & 1 
% \end{pmatrix}\\
% \end{equation*}
\begin{equation*}                     
%\text{ and} \quad
\Sigma_{\text{Toeplitz}} = \begin{pmatrix}
1      & \phi & \cdots & \phi^{d-1} \\ 
\phi & \ddots & \ddots & \vdots \\ 
\vdots & \ddots & \ddots & \phi \\
\phi^{d-1}  & \cdots & \phi & 1 
\end{pmatrix}
\end{equation*}

We run OOMP for $d \in \left\lbrace 2^2,  2^3, \dots, 2^8 \right \rbrace$, we average the number of queried quantities over 20 runs and plot the ratio $\frac{C^{\text{OOMP}}}{C^{\text{OMP}}}$ in the logarithmic scale with base 2 as a function of $\log_2 d$ (Figure~\ref{fig:complexity_comp}). We set $\delta = 0.1$. 
In all our simulation runs, the support $S^*$ was correctly recovered.
The results reported in Figure~\ref{fig:complexity_comp} show a significant reduction of the complexity between OOMP and OMP. %Moreover, the decay of the ratio is consistent with the decay given by \eqref{eq:specific_sc}, which corresponds to $\frac{1}{s^*} = \frac{1}{\log(d)}$. \todo{$1/d$ ou $1/\log(d)$?}
%The results reported in Figure~\ref{fig:complexity_comp}
%show that OOMP allowed for a very significant
%reduction of the complexity of the \tryselectproc step, which is the dominant part because of the role
%of the dimension $d$.
%Concerning the \optimproc step, we recall that we artificially reduced the number
%of iterations by a large factor with respect to the theoretical prescription. While the
%absolute comparison to OMP may be disputable for this reason, there is
%a trend in favor of OOMP as the dimension increases. Furthermore, due to the modular approach of our
%method, it is possible to take advantage of any new development on confidence bounds for
%SAGD (or variants thereof) and plug it directly into the \optimproc step. We note that uncertainty
%quantification (exact confidence bounds) for SAGD has only been tackled recently in the literature,
%so there certainly is room for further practical improvement on this side in the future,
%for example using the ideas developed by~\cite{pillaud2018exponential}.

\begin{figure}[H]
  \centering
  \includegraphics[height=90mm, width=90mm,scale=1]{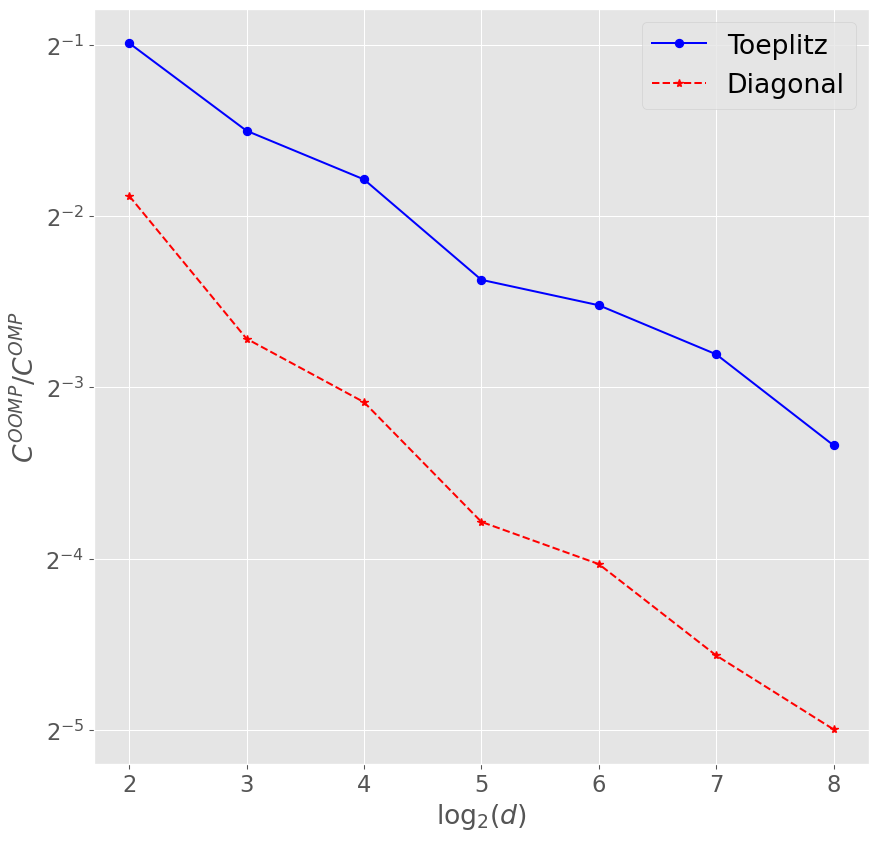}
  \caption{ \label{fig:complexity_comp} Comparison of computational complexities.
    \small
    The ratio $\frac{C^{\text{OOMP}_{.}}}{C^{\text{OMP}_{.}}}$ is plotted as a function of $\log_2\left(d\right)$ for both the Diagonal and Toeplitz covariance matrix. }
  \label{failure_oomp}
\end{figure}

\bibliographystyle{plainnat}
\bibliography{bib_data_base}

\appendix

\section{Proof of Lemma~\ref{lem:ineqzi}} 
Suppose Assumptions~\ref{ass:ass4} and~\ref{ass:ass3} hold. For any subset $S \subseteq [d]$ define $\beta^S := \argmin_{\text{supp}(\beta) \subseteq S} \mathcal{R}(\beta)$, with $\mathcal{R}(\beta) =  \mathbb{E}_{(x,y)}\left[\left(y-\langle x, \beta \rangle\right)^{2}\right]$ .

Let us fix $S \subseteq S^*$, recall that
$Z_i^S = \mathbb{E}\left[ x_i (y-x^{t}\beta^S) \right]$; at first we only use the fact that the support $S$
of $\beta^S$ is a subset of $S^*$. %\marginpar{G: notation: petit problème de compatibilités de dimensions entre $\Sigma_{S^*}, \beta_{S^*}$ et $\beta_S$}
We have, if $S^*\neq \emptyset$:
\begin{align*}
\max_{i \in S^*} |Z_i^S| = \underset{i \in S^*}{\max}\left|\text{Cov}\left( x_i, y-x^{t}\beta^S \right)\right| & =  \underset{i \in S^*}{\max} \left| \text{Cov} \left( x_i, x^{t}(\beta^{S^*}-\beta^S) \right) \right| \\
& =  \underset{i \in S^*}{\max} \left| \mathbb{E}\left[x_i x^{t} (\beta^{S^*}-\beta^S) \right] \right| \\
& =  \underset{i \in S^*}{\max} \left| \mathbb{E} \left[e_i^{t}xx^{t} (\beta^{S^*}-\beta^S) \right] \right|\\
& =  \underset{i \in S^*}{\max} \left| e_i^{t} \Sigma \left( \beta^{S^*}-\beta^S \right) \right|\\
& =  \norm[2]{\Sigma \paren[1]{ \beta^{S^*}-\beta^S }}_{\infty}.
\end{align*}
(The above remains true for $S^*=\emptyset$ with the convention $\max \emptyset=0$).
Recall that $S \subseteq S^*$, 
hence the support of $\beta_S$ is included in $S^*$. Moreover by definition of $\beta^{S^*}$, its support is in $S^*$. Therefore, we have:
%\begin{align*}
%\max_{i \in S^*} |Z_i^S| = \underset{i \in S^*}{\max}\left|\text{Cov}\left( x_i, y-x^{t}\beta^S \right)\right| & =  \underset{i \in S^*}{\max} \left| \text{Cov} \left( x_i, x^{t}(\beta^{S^*}-\beta^S) \right) \right| \\
%& =  \underset{i \in S^*}{\max} \left| \mathbb{E}\left[x_i x^{t} (\beta^{S^*}-\beta^S) \right] \right| \\
%& =  \underset{i \in S^*}{\max} \left| \mathbb{E} \left[e_i^{t}xx^{t} (\beta^{S^*}-\beta^S) \right] \right|\\
%& =  \underset{i \in S^*}{\max} \left| e_i^{t} \Sigma \left( \beta^{S^*}-\beta^S \right) \right|\\
%& =  \|\Sigma \left( \beta^{S^*}-\beta^S \right) \|_{\infty}.
%\end{align*}
%
%Recall that $S \subseteq S^*$, hence the support of $\beta_S$ is included in $S^*$. Moreover by definition of $\beta^{S^*}$, its support is in $S^*$. Therefore, we have:
\begin{equation*}
\max_{i \in S^*} |Z_i^S| = \|\Sigma_{S^*} \left( \beta^{S^*}_{S^*}-\beta^S_{S^*} \right) \|_{\infty}.
\end{equation*} 
Let $v = \Sigma_{S^*} \left( \beta^{S^*}_{S^*}-\beta^S_{S^*} \right)$, and assume $v\neq 0$ (the case $v=0$ is trivial). By definition of $\mu_{S^*}$, we have
for any $j\notin S^*$, using Assumption~\ref{ass:ass3} and the previous display:
\begin{align*}
\mu_{S^*} &= \underset{j \notin S^*}{\max} \left\|\Sigma_{S^*}^{-1} \text{Cov}\left( x_{S^*}, x_j \right)\right\|_1 \\
& \ge  \frac{ \left| \text{Cov}(x_{S^*}, x_j)^{t}\Sigma_{S^*}^{-1}v \right| }{\|v\|_{\infty}} \\
&= \frac{ \left| \text{Cov}(x_{S^*}, x_j)^{t}(\beta_{S^*}^{S^*}-\beta^S_{S^*}) \right| }{\|v\|_{\infty}} \\
&= \frac{ \left| \mathbb{E} \left[ x_j x_{S^*}^{t}(\beta_{S^*}^{S^*}-\beta^S_{S^*})\right] \right|}{\|v\|_{\infty}} \\
&= \frac{ \left| \mathbb{E} \left[ x_j (y-x^{t}\beta^S) \right] \right| }{\|v\|_{\infty}}\\
& = \frac{|Z_j^S|}{\max_{i\in S^*} |Z_i^S|}.
%&= \frac{|\text{Cov}(x_j, Res^{(k-1)})|}{\|v\|_{\infty}}
\end{align*}
We now use the actual definition of $\beta^S$, namely $\beta^S = \argmin_{\text{supp}(\beta) \subseteq S} \mathcal{R}(\beta)$, with $\mathcal{R}(\beta) =  \mathbb{E}_{(x,y)}\left[\left(y-\langle x, \beta \rangle\right)^{2}\right]$. Since $\partial_i \mathcal{R}(\beta) = -2 \mathbb{E}_{(x,y)}\left[x_i \left(y-\langle x, \beta \rangle\right)\right]$, we must have $ 0 = \partial_i \mathcal{R}(\beta^S) = -2Z_i^S$ for all $i \in S$.
% Finally recall that for all $i \in F^{(k-1)}$, using $\beta^{(k-1)} = \underset{\beta \in F^{(k-1)}}{\argmin} \mathbb{E}[(y-x^{T}\beta)^{2}]$, differentiating with respect to a feature $i$ in the support $F^{(k-1)}$, we have: $\text{Cov}(x_i, Res^{(k)}) = 0$, which means that: $\underset{i \in F^{*}}{\max} |Z_{i}^{(k)}| = \underset{i \in F^{*} \setminus F^{(k-1)}}{\max} |Z_{i}^{(k)}|$.
% From the calculations and the observations above we deduce the result.
We conclude that $\max_{i\in S^*} |Z_i^S| = \max_{i\in S^*\setminus S} |Z_i^S|$
(including in the case $S=S^*$ where the latter right-hand side is 0 by convention),
yielding the desired conclusion in conjunction
with the last display.
%We therefore have, for all $j \notin S^*$: $ |Z_j^S|<\mu_{S^*} \underset{i \in S^* \setminus S}{\max} |Z_{i}^S|$.

\section{Technical Results}\label{sec:tecres}

In this section we collect some technical results we will need for the proofs below.
Recall that we assume the exact linear model: 
\begin{equation*}
y=\langle x,\beta^{S^*} \rangle+\epsilon,
\end{equation*}
with $\mathbb{E}[\epsilon|x]=0$. %Let us fix $S \subsetneq S^*$, such that $S \neq \emptyset$. We can write our model as: $y=\langle x,\beta_S \rangle+z$ with $z=\langle x,\beta^{*}-\beta_S\rangle+\epsilon$, and consider that the ambient dimension is $k:=|S|$ (i.e $x\in \mathbb{R}^{k}$, $\beta_S \in \mathbb{R}^{k}$).
In the result to come we restrict our attention to vectors $\beta$ having support included in $S$ for a fixed $S \subseteq S^*$ and denote $k:=\left| S\right|$.
Consequently we can with some abuse of notation assume that the ambient dimension is reduced to $k$ (i.e $x\in \mathbb{R}^{k}$, $\beta^S \in \mathbb{R}^{k}$);
let us denote by $\mathcal{R}:\mathbb{R}^{k} \to \mathbb{R}$ the loss function defined by: $\mathcal{R}(\beta) = \mathbb{E}[(y-x^{t}\beta)^2]$, $g:\mathbb{R}^{k} \to \mathbb{R}^{k}$ the gradient function defined by $g(\beta) = \nabla \mathcal{R}(\beta) = \mathbb{E}[2(x^{t}\beta- y)x]$ and for a sample $(x,y)$ define: $\hat{g}_{(x,y)}(\beta) = 2(x^{t}\beta- y)x$. Denote by $\mathcal{B}_{k} \left( 0,r \right) $ the closed ball centred at the origin with radius $r$ in $\mathbb{R}^{k}$.    

%\marginpar{G: pourquoi $S\neq \emptyset$ nécessaire? Pourquoi $S\neq S^*$?}

\begin{lemma}\label{lem:tech}
	Suppose Assumptions~\ref{ass:ass1} and~\ref{ass:ass2} hold. Considering
	the restrictions of functions $g,\hat{g},\mathcal{R}$ to vectors $\beta$ having support in $S^*$ and
	reducing implicitly the ambient dimension to $s^*=|S^*|$, we have:
	\begin{enumerate}%[label={\upshape(\roman*)}, align=left, widest=iii, leftmargin=1cm]
		%\item $|\epsilon| \le 2M$    (a.s).
		\item for any $S \subseteq S^*$: $\left\| \beta^S \right\|_2 \le \frac{2}{\sqrt{\rho}}$.
		\item $\forall \beta \in \mathcal{B}_{k} \left( 0,\frac{2}{\sqrt{\rho}} \right) $:  $\left\| \hat{g}_{(x,y)}(\beta)\right\|_2 \le 4k\frac{M^{2}}{\sqrt{\rho}}+2\sqrt{k}M$  (a.s).
		%\item $\forall \beta \in \mathcal{B}_{s^*} \left( 0,\frac{2}{\sqrt{\rho}} \right)$: $\left\| g(\beta)\right\|_2 \le 6\frac{L}{\sqrt{\rho}}$.
		\item $\forall \beta \in \mathcal{B}_{k} \left( 0,\frac{2}{\sqrt{\rho}} \right)$: $\left\| g(\beta)\right\|_2 \le 4k\frac{M^{2}}{\sqrt{\rho}}+2\sqrt{k}M$.
		\item $\mathcal{R}:\mathbb{R}^{k} \rightarrow \mathbb{R}$ is $\rho$-strongly convex. 
	\end{enumerate}
\end{lemma}

\begin{proof} Recall that from Assumption~\ref{ass:ass1}, then
	the eigenvalues of the matrix $\Sigma_{S^*}$ belong to $[\rho,L]$.

	\begin{enumerate}%[label={\upshape(\roman*)}, align=left, widest=iii, leftmargin=1cm]
		\item %$\| \beta_S\|_2 \le \frac{2}{\sqrt{\rho}}$:	
		Since $\mathbb{E}[\epsilon | x] = 0$, and $y = x^t \beta^{S^*}+\epsilon$, we have
		for any $S\subseteq S^*$:
		\begin{equation*}
		\mathbb{E}\left[ \left( y-x^{t}\beta^S \right)^{2}\right] = \mathbb{E} \left[ \left(x^{t}\left( \beta^{S^*}-\beta^S \right) \right)^{2}\right]+ \mathbb{E} \left[ \epsilon^{2} \right].
		\end{equation*}	
		By definition of $\beta^S$, it holds $\mathbb{E}\left[ \left(y-x^{t}\beta^S \right)^{2} \right] \le \mathbb{E}\left[ y^{2} \right] \le 1$, together with the above it gives:
		\begin{equation*}
		\rho \|\beta^{S^*}-\beta^S\|_2^{2} \leq
		\left( \beta^{S^*}-\beta^S \right)^{t} \Sigma_{S^*} \left( \beta^{S^*}-\beta^S \right)
		= \mathbb{E}\left[\left( x^{t} \left( \beta^{S^*}-\beta^S \right) \right)^{2} \right]\le 1.
		\end{equation*}
		In particular for
		$S = \emptyset$, we have: $\| \beta^{S^*} \|_2 \le \frac{1}{\sqrt{\rho}}$.	
		%We conclude using the upper bound of $\|\beta_{S^*}\|_2$,
		% that
		By the triangle inequality, for an arbitrary $S \subseteq S^*$:	
		\begin{equation*}
		\| \beta^S \|_2 \le \frac{2}{\sqrt{\rho}}.
		\end{equation*}

		\item %$\forall \beta \in \mathcal{B}_k \left( 0,\frac{2}{\sqrt{\rho}} \right) $:  $\| \hat{g}(\beta)\|_2 \le 4k\frac{M^{2}}{\sqrt{\rho}}+2\sqrt{k}M$  (a.s):	
		Let $\beta \in \mathcal{B}_{k} \left( 0,\frac{2}{\sqrt{\rho}} \right)$, we have:
		\begin{align*}
		\|\hat{g}_{(x,y)}(\beta)\|_2 = \| 2(x^{t}\beta-y)x\|_2 &\le |2x^{t}\beta| \|x\|_2+2|y| \|x\|_2 \\
		&\le 2  \|\beta\|_2 \|x\|_2^2 + 2|y|\|x\|_2\\
		&\le 2k \|x\|_{\infty}^{2} \|\beta\|_2 + 2\sqrt{k}\|x\|_{\infty}\\
		&\le 4k \frac{M^{2}}{\sqrt{\rho}}+2\sqrt{k}M;
		\end{align*}
		where we used: $\|x\|_2 \le \sqrt{k}\|x\|_{\infty}$, and the assumptions $\|x\|_{\infty} \le M$, $|y| \le 1$.

		\item %$\forall \beta \in \mathcal{B}_k \left( 0,\frac{2}{\sqrt{\rho}} \right)$: $\left\| g(\beta)\right\|_2 \le 6\frac{L}{\sqrt{\rho}}$:
		Let $\beta \in \mathcal{B}_{k} \left( 0,\frac{2}{\sqrt{\rho}} \right)$, we have:
		%Recall that by definition: $g(\beta) = \mathbb{E}\left[2(x^{t}\beta - y)x \right]$.
		%	
		%	We have:
		\begin{align*}
		\left\|g(\beta)\right\|_2  &= \left\|\mathbb{E}\left[ \hat{g}_{(x,y)}(\beta)\right] \right\|_2\\
		&\le \mathbb{E} \left[ \|\hat{g}_{(x,y)}(\beta)\|_2\right] \\
		&\le 4k \frac{M^{2}}{\sqrt{\rho}}+2\sqrt{k}M;
		%	&\le \left\| 2\Sigma\beta - 2\mathbb{E}\left[ yx \right] \right\|_2\\
		%	&\le 2\left\|\Sigma \beta \right\|_2 + 2\left\|\mathbb{E}[(x^{t}\beta^{S^*}+\epsilon)x]\right\|_2\\
		%	&\le 2\left\|\Sigma \beta \right\|_2 + 2 \left\|\Sigma \beta^{S^*} \right\|_2\\
		%	&\le 2L \frac{2}{\sqrt{\rho}}+2L \frac{1}{\sqrt{\rho}} = 6\frac{L}{\sqrt{\rho}},
		\end{align*}
		using the estimate of the previous point.
		%	We conclude that:
		%	\begin{equation*}
		%		\left\|g(\beta)\right\|_2 \le 6\frac{L}{\sqrt{\rho}}
		%	\end{equation*}
		
		\item %$\mathcal{R}$ is $\rho$-strongly convex:	
		Recall that $\mathcal{R}$ is twice differentiable and its Hessian is given by $\mathbb{E}[xx^{t}] =  \Sigma_{S^*}\geq \rho I_{s^{*}}$,
		therefore $\mathcal{R}$ is $\rho$-strongly convex.
	\end{enumerate}
\end{proof}

\section{Proof of Lemma~\ref{global-lem}}

Let us start by restating Lemma~\ref{global-lem}.
\begin{lemma}
	Suppose that Assumptions~\ref{ass:ass3} and~\ref{ass:ass4} hold. Consider Algorithm~\ref{algo:oomp} with the procedure \selectproc given in Algorithm~\ref{algo:step}, assume that \optimproc satisfies the {\em optimization confidence property} and that \tryselectproc satisfies the {\em selection property}. Then when the \textbf{OOMP}($\delta, s^*$) (Algorithm~\ref{algo:oomp}) is terminated, the variable $S$ satisfies with probability at least $1-2\delta$: $S \subseteq S^*$.
\end{lemma}

\begin{proof}
	First consider an idealized setting where the algorithm runs indefinitely.
	Let $U_p$ denote the set of selected features at the $p$-th iteration of the main {\bf while}
	loop of Algorithm~\ref{algo:oomp}. It can happen that the call to {\bf Select}
	never terminates (this is actually the expected behaviour if all relevant features have been
	already discovered), so if $\bar{\tau}$ denotes the (random) last terminating iteration,
	we formally define $U_p=U_{\bar{\tau}}$ if $p>\bar{\tau}$ (this is of course irrelevant in practice
	but is just needed to always have a formally well defined $U_p$ for all integers $p$). Denoting $S_p:=\bigcup\limits_{i=1}^p U_i$,
	we see that with this definition, for any integer $k\geq 1$:
	\[
	\mathbb{P}\left(U_k \not\subset S^* | S_{k-1} \subseteq S^*\right)
	= \mathbb{P}\left(U_k \not\subset S^*; \bar{\tau} \geq k | S_{k-1} \subseteq S^* \right).
	\]
	The event $\bar{\tau}\geq k$ implies that all iterations including the $k^{th}$ one
	have terminated. Furthermore, the $k$th selection iteration then consisted in calling
	repeatedly the \tryselectproc  with allowed error probability $\delta_{k,i} = (k(k+1)2^i)^{-1} \delta$ at the $i$-th call,
	until it returned {\tt Success=true} (indicating termination of the $k$-th main selection iteration). Let us denote $B_{k,i}$ the event ``the $i$-th call to \optimproc  during
	the $k$-th selection iteration, if it took place, returned $\tilde{\beta}^S$ such that the optimization confidence property~\eqref{eq:condopt} holds'', and $A_{k,i}$ the event ``the $i$-th call to \tryselectproc  during
	the $k$-th selection iteration, if it took place, returned {\tt Success=true} and a subset of features $U \not\subset S^*$.''
	It holds $\mathbb{P}(B_{k,i}^c| S_{k-1} \subseteq S^*)\leq \delta_{k,i}$ by the optimization confidence property, and
	$\mathbb{P}(A_{k,i} | S_{k-1} \subseteq S^*,B_{k,i})\leq \delta_{k,i}$ by the selection property, so we have
	\begin{align*}
	\mathbb{P}(U_k \not\subset S^*; \bar{\tau} \geq k | S_{k-1} \subseteq S^*)
	&\leq \prob{\bigcup_{i=1}^\infty A_{k,i} \Big| S_{k-1} \subseteq S^*}\\
	&\leq \sum_{i=1}^\infty \mathbb{P}(A_{k,i} | S_{k-1} \subseteq S^*)\\
	&\leq \sum_{i=1}^\infty \mathbb{P}(A_{k,i} \cap B_{k,i} | S_{k-1} \subseteq S^*)+ \mathbb{P}(B_{k,i}^c | S_{k-1} \subseteq S^*) \\
	&\leq \sum_{i=1}^\infty \mathbb{P}(A_{k,i} | S_{k-1} \subseteq S^*, B_{k,i})+ \mathbb{P}(B_{k,i}^c | S_{k-1} \subseteq S^*) \\
	&\leq 2 \sum_{i=1}^\infty \delta_{k,i}.
	\end{align*}
	%  \marginpar{GB: (i) specify what are $\delta_{k,i}$\\(ii) take properly into account the optimization confidence property}
	%  using the selection propety of the \tryselectproc  procedure.
	
	Now, the algorithm may be interrupted at a completely arbitrary time, and returns the
	last active set $S=S_{\tau}$ for some $\tau \leq \bar{\tau}$. We then have
	\begin{align*}
	\prob{S_{\tau} \nsubseteq S^*}  \leq \prob{ \exists k\geq 1: S_k \nsubseteq S^*}
	&\leq \prob{ \exists k\geq 1: U_k \not\subset S^*; S_{k-1} \subseteq S^*}\\
	& \leq  \sum_{k\geq 1} \prob{ U_k \not\subset S^*; S_{k-1} \subseteq S^*}\\
	& \leq  \sum_{k\geq 1} \prob{ U_k \not\subset S^*| S_{k-1} \subseteq S^*}\\
	%     & \leq  \sum_{k\geq 1} \prob{ s_k \not\in S^*; \bar{\tau} \geq k| S_{k-1} \subseteq S^*}\\
	&\leq 2\sum_{k,i=1}^\infty \delta_{k,i} = 2\delta.
	\end{align*}
\end{proof}

\section{Proof of Proposition~\ref{prop:optim_prop}}

In this section we give high probability bounds on the output of the averaged stochastic gradient descent (ASGD, Algorithm~\ref{alg:gen_optim}). Theorem~\ref{thm:conoptim} below is a slight modification of the main result in \cite{harvey2019tight}, which consists in assuming that the error on the stochastic sub-gradients is bounded by a constant $G>0$ instead of $1$. We denote by $\Pi_{\mathcal{X}}$ the projection operator on $\mathcal{X} := \mathcal{B}\left(0, \frac{2}{\sqrt{\rho}}\right)$.

\begin{algorithm}[H]
	\caption{ASGD($T$, $\beta_0$)\label{alg:gen_optim}}
	\begin{algorithmic}
		\STATE {\bfseries Input:} initial  $\beta_0$, $T$
		\FOR{$t\gets 0,...,T-1$}
		\STATE $\eta_t \gets \frac{2}{\rho (t+1)}$, $\nu_t \gets \frac{2}{t+1}$
		\STATE $(X,Y) \gets \textbf{query-new}(S\cup \{d+1\})$
		\STATE $\gamma_{t+1} \gets \beta_{t} - 2\eta_t (X^{t}\beta_{t}-Y)X$
		\STATE $\beta_{t+1} \gets \Pi_{\mathcal{X}}(\gamma_{t+1})$
		\STATE $\tilde{\beta}_{t+1} \gets (1-\nu_{t})\tilde{\beta}_{t}+\nu_{t} \beta_{t+1}$
		\ENDFOR
		\STATE $\textbf{return   } \tilde{\beta}_T$
	\end{algorithmic}
\end{algorithm}

We use the same notations as in Section~\ref{sec:tecres}, we assume with some abuse of notation that the ambient dimension is reduced to $k:=|S|$ (i.e $x\in \mathbb{R}^{k}$, $\beta^S \in \mathbb{R}^{k}$).
We recall that we denote by $\mathcal{R}:\mathbb{R}^{k} \to \mathbb{R}$ the loss function defined by: $\mathcal{R}(\beta) = \mathbb{E}[(y-x^{t}\beta)^2]$, $g:\mathbb{R}^{k} \to \mathbb{R}^{k}$ the gradient function defined by $g(\beta) = \nabla \mathcal{R}(\beta) = \mathbb{E}[2(x^{t}\beta- y)x]$; in addition we consider $ \hat{g}_{n}:\mathbb{R}^{k} \to \mathbb{R}^{k}$ defined by $\hat{g}_{n}(\beta) = 2((x_S^{(n)})^{t}\beta-y^{(n)})x^{(n)}$, where $(x^{(n)},y^{(n)})$ are the output of the $n^{th}$ call of $\textbf{query-new}$ during Algorithm~\ref{alg:optim}. Denote by $\mathcal{B}_{k} \left( 0,r \right) $ the closed ball centred at the origin with radius $r$ in $\mathbb{R}^{k}$.

%Taking $\mathcal{X} = \mathcal{B}_{k}(0, \frac{2}{\sqrt{\rho}})$,
Lemma~\ref{lem:tech} shows that (under Assumptions~\ref{ass:ass1}-\ref{ass:ass2}), we have
via the triangle inequality:
\begin{equation}\label{eq:g}
\| \hat{g}_{t+1}(\beta_{t}) - g(\beta_t) \| \le 8k\frac{M^2}{\sqrt{\rho}}+4\sqrt{k}M .
\end{equation}
Where $\beta_t$ are the iterates of Algorithm~\ref{alg:optim}. We denote by $G$ the upper bound in equation~(\ref{eq:g}). 

\begin{theorem}\label{thm:conoptim}
	Suppose Assumptions~\ref{ass:ass1} and~\ref{ass:ass2} hold. Let $\delta \in (0,1)$ and $S \subseteq S^*$ such that $S \neq \emptyset$. Denote by $\tilde{\beta}_T$ the output of $\text{ASGD}( T, 0)$ (Algorithm~\ref{alg:gen_optim}).
	
	Then, with probability at least $1-\delta$ with respect to the samples queried during Algorithm~\ref{alg:gen_optim}:
	\begin{equation*}
	\mathcal{R}(\tilde{\beta}_T) - \mathcal{R}(\beta^S) \le \frac{21G^2 \log(1/\delta)}{\rho T},
	\end{equation*}
	%where $G := 4k\frac{M^2}{\sqrt{\rho}}+2\sqrt{k}M + \frac{6L}{\sqrt{\rho}}$.
	where $G := 8k\frac{M^2}{\sqrt{\rho}}+4\sqrt{k}M $. 
\end{theorem}

\vspace{4mm}
The following corollary results by simply choosing $T$ large enough such that the optimization confidence property is satisfied by Algorithm~\ref{alg:gen_optim}.
\begin{corollary}
	Suppose assumptions Suppose Assumptions~\ref{ass:ass1} and~\ref{ass:ass2} hold. Let $\xi >0, \delta \in (0,1)$. Consider algorithm~\ref{alg:gen_optim} with inputs $(T,0)$ such that:
	\begin{equation*}
	T = \frac{21 G^2 \log\left(1/\delta\right)}{\rho \xi},
	\end{equation*}
	where $k:= \left|S\right|$ and $G:= 8k\frac{M^2}{\sqrt{\rho}}+4\sqrt{k}M$. Then the output $\tilde{\beta}_T$ satisfies with probability at least $1 - \delta$:
	\begin{equation*}
	\mathcal{R}(\tilde{\beta}_T) - \mathcal{R}(\beta^S) \le \xi.
	\end{equation*}
\end{corollary}

\section{Proof of Proposition~\ref{prop:conczi}}

\subsection{Technical Results} \label{subsec:tech}

The following result is a straightforward modification of the empirical Bernstein inequality from \cite{DBLP:conf/colt/MaurerP09}, which consists in assuming that the random variables $U_i$ belong to $[-B,B]$ for a $B>0$, instead of $[0,1]$.

\begin{lemma}{ \cite{DBLP:conf/colt/MaurerP09}}\label{lem:empber}
	Let $U, U_1,\ldots, U_n$ be i.i.d. random variables with values in $[-B,B]$ and let $\delta > 0$. Then with probability at least $1-\delta$ we have: 
	\begin{equation*}
	\left| \frac{1}{n} \sum_{i=1}^{n} U_i - \mathbb{E}\left[ U \right] \right|  \le \sqrt{\frac{2V_n \ln(2/\delta)}{n}}+\frac{14B\ln(2/\delta)}{3(n-1)},
	\end{equation*}
	where:
	\begin{equation*}
	V_n = \frac{1}{n(n-1)}\sum_{1\le i < j \le n} (U_i-U_j)^{2}.
	\end{equation*}
\end{lemma}

%\paragraph{Remark:}  The concentration inequality given in \cite{maurer2009empirical} used a different sample variance defined by $V_n = \frac{1}{n(n-1)} \sum_{1\le i<j \le n} (U_i-U_j)^2$, which gives a tighter bound. We upper bounded this quantity by the expression of $V_n$ presented in the lemma above because it can be computed more efficiently. 

We are interested in applying the Lemma above to the quantities $\tilde{Z}_{i,n}^S$. Let $\left(X,Y\right)$ be a queried sample, the following claim shows that the random variable $U := X_{i} (X^{t} \tilde{\beta}_S - Y)$ for $i\in [d]$, where $X_{i}$ is the $i^{th}$ feature $X$, satisfies the conditions of Lemma~\ref{lem:empber}.

\begin{claim}\label{claim:bound}
	Suppose Assumption~\ref{ass:ass2} holds. Let $(X,Y)$ be a sample, $\beta \in \mathbb{R}^d$ of support $S \subseteq [d]$ and such that $\|\beta\|_2 \le \frac{2}{\sqrt{\rho}}$. Fix $i \in [d]$ and define $U = X_{i} \left( X^t \beta - Y\right)$.
	Then it holds almost surely:
	\begin{equation*}
	\left| U\right| \le 2\sqrt{\frac{|S|}{\rho}} M^2 + M.
	\end{equation*}
\end{claim}

\begin{proof}
	Using the Cauchy-Schwartz inequality, we have:
	\begin{align*}
	\left| U \right| &\le |X_{i}| \left( \|X_S\| \|\beta\| + |Y|\right) \\
	& \le M \left( \sqrt{|S|}M \frac{2}{\sqrt{\rho}} +1\right).
	\end{align*}
	
\end{proof}
Moreover, a straightforward calculation yields the result below.
\begin{claim}\label{claim:boundprac}
	Suppose Assumption~\ref{ass:ass2} holds. Let $(X,Y)$ be a sample, $\beta \in \R^d$ of support $S \subseteq [d]$. Fix $i \in [d]$ and define $U := X_{i} \left( X^t \beta - Y\right)$.
	Then it holds
	\begin{equation*}
	\left| U\right| \le M^2 \|\beta\|_1 + M.
	\end{equation*}
\end{claim}
\begin{proof}
	We have:
	\begin{align*}
	\left| U \right| &\le \left| X_i\right| \left( \norm{X}_{\infty} \norm{\beta}_1 + \left| Y\right|_{\infty} \right) \\
	& \le M \left( M \norm{\beta}_1 + 1\right).
	\end{align*}
\end{proof}

\subsection{Proof of Proposition~\ref{prop:conczi}}

Consider an i.i.d sequence $\left(X_h, Y_h\right)$. Let $n \ge 1$ and denote $\left(X_h, Y_h\right)_{1\le h \le n}$ in matrix and vector form as: $\bm X \in \mathbb{R}^{n \times d}, \bm Y \in \mathbb{R}^{n}$. 

Let us first fix a set $S \subseteq S^*$, a feature $i \in [d]\setminus S$ and a vector $\beta \in \mathbb{R}^d$.
Denote for all $j \in [n]$: $U_{j} := \bm{X}_{j,i} (\bm{X}_j^{t} \beta - \bm{Y}_j)$, where $\bm{X}_{j,i}$ is the $i^{th}$ feature of the $j^{th}$ sample $\bm{X}_j$. Recall that  $\tilde{Z}_{i,n}^S(\beta) = \frac{1}{n} \sum_{j=1}^{n} U_{j}$ and $Z_i^S = \mathbb{E}_{(x,y)} [x_i \left( x^t\beta^S-y \right)]$ . We have:

\begin{align*}
\left|\tilde{Z}_{i,n}^S(\beta)-Z_i^{S} \right| &= \left|\frac{1}{n} \sum_{j=1}^{n} U_{j}- \mathbb{E}_{(x,y)} \left[ x_i \left( x^t\beta^S-y \right) \right] \right| \\
&\le \left|\frac{1}{n} \sum_{j=1}^{n} U_{j}- \mathbb{E}_{(x,y)} \left[ x_i \left( x^t\beta-y \right) \right] \right|+ \left|\mathbb{E}_{(x,y)} \left[ x_i \left( x^t\beta-y \right) \right]-\mathbb{E}_{(x,y)} \left[ x_i \left( x^t\beta^S-y \right) \right] \right| \\
&\le \left| \frac{1}{n} \sum_{j=1}^{n} U_{j} - \mathbb{E}_{(x,y)}\left[ U_{1}\right] \right|+ \left|\mathbb{E}_{(x,y)} \left[ x_i x^t \left( \beta-\beta^S \right) \right] \right| \\
&\le \left| \frac{1}{n} \sum_{j=1}^{n} U_{j} - \mathbb{E}_{(x,y)}\left[ U_{1}\right] \right|+ M\left|\mathbb{E}_{(x,y)} \left[ \left| x^t \left( \beta-\beta^S \right)\right|  \right] \right| \\
&\le \left|\frac{1}{n} \sum_{j=1}^{n} U_{j} - \mathbb{E}_{(x,y)}[U_{1}] \right| +M\sqrt{\mathcal{R} \left(\beta \right)-\mathcal{R}\left(\beta^S \right)}.
\end{align*}

Let us denote $\tilde B(\beta):= M^2 \norm{\beta}_1 + M$, and $\tilde V_n(\beta) := \frac{1}{n(n-1)}\sum_{1\le p < q \le n}(U_{q}-U_{p})^{2}$. Since $(U_{j})_{j \in [n]}$ are i.i.d and belong to $[-B,B]$ (Claim~\ref{claim:boundprac}, following from
Assumption~\ref{ass:ass1} and Lemma~\ref{lem:tech}~(i)), we have using Lemma~\ref{lem:empber}: for any $\delta \in (0,1)$, with probability at least $1-\frac{\delta}{4dn^2}$:

\begin{equation}\label{eq:empber}
\left|\frac{1}{n} \sum_{j=1}^{n} U_{j} - \mathbb{E}_{(x,y)}\left [U_{1} \right] \right| \le \sqrt{\frac{2\tilde V_n(\beta) \log(8dn^2/\delta)}{n}}+\frac{14\tilde B(\beta)\log(8dn^2/\delta)}{3(n-1)}.
\end{equation}

Now we apply a union bound over the sample size $n\ge 1$ and features $i \in [d]\setminus S$, we obtain: with probability at least $1-\frac{\delta}{2}$, bound~\eqref{eq:empber} holds for all $n$ and $i$. To conclude, we choose $\beta = \tilde{\beta}^S$ and  we use the risk bound~\eqref{eq:condopt} to have: with probability at least $1-\delta$:
\begin{equation*}
\forall i \in [d], \forall n\geq 1: \qquad
\left|\tilde{Z}_{i,n}^S(\tilde{\beta}^S)-Z_i^{S} \right| \le \sqrt{\frac{2\tilde V_n(\tilde{\beta}^S) \log(8dn^2/\delta)}{n}}+\frac{14\tilde B(\tilde \beta^S)\log(8dn^2/\delta)}{3(n-1)} + M \sqrt{\xi} .
\end{equation*}   
%Now if $\beta=\tilde{\beta}_S$ satisfies the risk bound~\eqref{eq:condopt}, combining with the above
Recall:
\begin{equation*}
\tilde V_n^{+}(\beta) := \max\left( \tilde V_n(\beta), \frac{1}{1000} \frac{LM^2}{\rho}\right).
\end{equation*}
Using the fact that $\tilde V_n(\beta) \le \tilde V_n^{+}(\beta)$, combining with the above
inequality % and a union bound over $i\in[d]$,
we get the announced claim.

\section{Detailed algorithms for \tryselectproc}\label{sec:algos}

Algorithm~\ref{algo:select_mc_full} is a detailed version of Algorithm~\ref{algo:select_mc} 
(the shortened version in the main body of the paper).

%\subsection{Data Stream \tryselectproc }\label{algo:full_try_sel_ds}
%(See Algorithm~\ref{algo:select_mc})
\begin{algorithm}
	\caption{\tryselectproc($S$, $\delta$, $\tilde{\beta}$, $\xi$), Data Stream setting \label{algo:select_mc_full}}
	\begin{algorithmic}
		\STATE {\bfseries Input:} $S$, $\delta$, $\tilde{\beta}$, $\xi$
		\STATE {\bfseries Output:} $S$, $\text{Success}$
		\STATE let $n \gets 0$ be the number of queried samples.
		\STATE let $v \gets 0$ be an array to store the quantities $\tilde{V}_{i,n}$.
		\STATE let conf be an array to store the confidence bound values.
		\STATE let $Z$ be an array to store the quantities $\tilde{Z}_{i,n}^S$.
		\STATE let $U \gets \emptyset$ denote the set of selected variables.
		\STATE let $L \gets [d+1]\setminus S$ denote the set of candidate variables.
		\STATE //\textsc{beginning of initialization}
		\STATE $n \gets 1$
		\STATE $(X,Y) \gets \textbf{query-new}([d+1])$
		\STATE $\tilde{Z}_i \gets X_i \left(Y-X_S^{t}\tilde{\beta}\right)$, for all $i\in [d]\setminus S$.
		\STATE //\textsc{initialization for empirical variance quantities}
		\STATE $s_i \gets 0$, $m_i \gets X_i$, for all $i\in [d]\setminus S$.  
		\STATE //\textsc{ end of initialization}
		\WHILE{True}
		\STATE $(X,Y) \gets \textbf{query-new}([d+1])$
		\STATE $n \gets n+1$
		\STATE $\forall$ $i$: $Z_i \gets X_i \left(Y-X_S^t\tilde{\beta}\right)$
		\STATE $\forall$ $i$: $\tilde{Z}_i \gets \frac{1}{n} Z_i + \frac{n-1}{n} \tilde{Z}_i$.
		\STATE //\textsc{ updating the empirical variance}
		\STATE $\forall$ $i$: $\text{temp}_i \gets m_i$
		\STATE $\forall$ $i$: $m_i \gets m_i + (Z_i - m_i)/n_i$
		\STATE $\forall$ $i$: $ s_i \gets s_i + (Z_i - \text{temp}_i)*(Z_i - m_i)$
		\STATE $\forall$ $i$: $ v_i \gets s_i/\left(n_i-1\right)$
		\STATE $\forall$ $i$: $\text{conf}(i) \gets \sqrt{\frac{8v_i \log(8dn^2/\delta)}{n_{i}}} + \frac{28 B\log(8dn^2/\delta)}{3\left( n_{i} - 1\right)}$
		\IF {$2M \sqrt{\xi} > \min_i \{\text{conf}(i)\}$}
		\STATE $\text{Success} \gets \text{False}$, \textbf{break}
		\ENDIF
		\STATE let $\hat{i} \gets \underset{i \in [d]\setminus S}{\text{argmax}}\{ |\tilde{Z}_{i}|+\text{conf}(i) \}$
		\STATE //\textsc{Communicating an upper bound on the mean of the non-recovered coefficients}
		\STATE \textbf{ Communicate:} $ \sqrt{\frac{L}{\rho^3} \left( |\tilde{Z}_{\hat{i}}|+\text{conf}(\hat{i})\right)}$
		\FORALL {$i \in  L \setminus \{ d+1 \}$}
		\IF {$\left|Z_i\right| + \text{ conf}(i) \le \left|Z_{\hat{i}}\right| - \text{ conf}(\hat{i})$}
		\STATE $L \gets L \setminus \{i\}$
		\ENDIF
		\IF {$\left|Z_i\right| - \text{ conf}(i) \ge \mu \left( \left|Z_{\hat{i}}\right| + \text{ conf}(\hat{i})\right)$}
		\STATE $U \gets U \cup \{i\}$
		\ENDIF
		\ENDFOR
		\IF {$|\tilde{Z}_{\hat{i}}|> \frac{2}{1-\mu} \text{ conf}(\hat{i})$}
		\STATE $\text{Success} \gets \text{True}$, \textbf{break}
		\ENDIF
		
		\ENDWHILE
		\STATE $\textbf{return  } U, \text{Success}$
	\end{algorithmic}
\end{algorithm}

\paragraph{On the upper bound of the mean of the non-recovered coefficients:}
The bound communicated through the command:
\begin{equation*}
\textbf{ Communicate:} \quad\sqrt{\frac{L}{\rho^3} \left( |\tilde{Z}_{\hat{i}}|+\text{conf}(\hat{i})\right)}
\end{equation*}
Is a direct consequence of the bound in lemma~\ref{lem:ord} along with proposition~\ref{prop:conczi}.

\section{Proof of the selection property}

The proof that the proposed Algorithm~\ref{algo:select_mc} satisfies the selection property hinges on the following lemma:
\begin{lemma}
	\label{lem:selprop} 
	Let $S\subseteq S^*$ be fixed. Let $(\tilde{\beta}^S)$ be given. Assume there exists $n\ge1$, $\selarm, j \in [d] \setminus S$ and positive numbers $(\eps_i)_{i\in [d]\setminus S}$ are such that:
	\begin{align}
	\selarm & \in \mathrm{Argmax}_{i \in [d]\setminus S} \{ | \tilde{Z}^S_{i,n}| + \eps_i \}; \label{eq:selarm}\\
	\forall i \in [d]\setminus S: |\tilde{Z}_{i,n}^{S}-Z_{i}^{S}| & \le \eps_i; \label{eq:unifcont} \\
	|\tilde{Z}_{j,n}^S|- \eps_{j} & \ge \mu \left(|\tilde{Z}_{\selarm,n}^S|+\eps_{\selarm} \right). \label{eq:condarm}
	\end{align}
	% Then $\selarm \in S^* \setminus S$.
	Then it holds $\abs[1]{Z_{j}^S} \ge \mu \max_{i \in S^*} \abs{Z_{i}^S} $.
\end{lemma}
%\begin{lemma}
%	\label{lem:selprop} 
%	Let $S\subseteq S^*$ be fixed. Let $(\tilde{\beta}^S)$ and positive numbers 
%	$(n_i)_{i \in [d]\setminus S}$ be given. Assume there exists $\selarm \in [d] \setminus S$ and positive numbers $(\eps_i)_{i\in [d]\setminus S}$ are such that:
%	\begin{align}
%	\selarm & \in \mathrm{Argmax}_{i \in [d]\setminus S} \{ | \tilde{Z}^S_{i,n_i}| + \eps_i \}; \label{eq:selarm}\\
%	\forall i \in [d]\setminus S: |\tilde{Z}_{i,n_i}^{S}-Z_{i}^{S}| & \le \eps_i; \label{eq:unifcont} \\
%	|\tilde{Z}_{\selarm,n_i}^S| & > \frac{1+\mu_{S^{*}}}{1-\mu_{S^{*}}} \eps_{\selarm}. \label{eq:condarm}
%	\end{align}
%	% Then $\selarm \in S^* \setminus S$.
%	Then it holds $\abs[1]{Z_{\selarm}^S} > \mu_{S^*} \max_{i \in S^*} \abs{Z_{i}^S} $.
%\end{lemma}
\begin{proof}
	First assume $S\subsetneq S^*$. Let $\bestarm \in \mathrm{Argmax}_{i \in [d]\setminus S} \{ | Z^S_{i}| \} $. We have:
	
	\eqref{eq:selarm} implies that:
	\begin{equation*}
	|\tilde{Z}_{i^*,n}^S|+\eps_{i^*} \le |\tilde{Z}_{\selarm,n}^S|+\eps_{\selarm}
	\end{equation*}
	Moreover, using \eqref{eq:unifcont} twice along with \eqref{eq:condarm}:
	\begin{equation*}
	\left|Z_j^S\right| \ge \left|\tilde{Z}_{j,n}^S\right|-\eps_j \ge \mu \left(|\tilde{Z}_{\selarm,n}^S|+\eps_{\selarm} \right) \ge \mu \left(|\tilde{Z}_{\bestarm,n}^S|+\eps_{\bestarm} \right) \ge \mu\left|Z_{\bestarm}^S\right|
	\end{equation*} 
	
	In the case $S= S^*$, we have that $Z_i^S=0$ for all $i$,
	Therefore the claimed conclusion holds.
\end{proof}
Since Proposition~\ref{prop:conczi} ensures that~\eqref{eq:unifcont} is satisfied with probability 
$1-\delta$ (for $\eps_i=\mathrm{conf}(i,n_i,\delta)$, and uniformly for all values of $n_i$),
provided $2M \sqrt{\xi} < \text{conf}(i,n,\delta)$ for all $i$, Algorithm~\ref{algo:select_mc}, which checks the latter condition and selects $j$ satisfying~\eqref{eq:condarm},
satisfies the selection property.

\section{Proof of Lemma~\ref{lem:tau}}

Lemma~\ref{lem:tau} shows that the procedure \selectproc given in Algorithm~\ref{algo:step}, where \tryselectproc is given by  Algorithm~\ref{algo:select_mc} in the Data Stream setting and \optimproc given by Algorithm~\ref{alg:optim}, finishes in finite time if $S \subsetneq S^*$ and with high probability doesn't select any feature if $S = S^*$.

We start by stating the two following technical claim.  

\begin{claim}\label{lem:evid}
	Let Assumptions \ref{ass:ass4} and \ref{ass:ass3} hold, and $S \subsetneq S^*$.	
	Then $\max_{i \in [d]\setminus S} \{ \left| Z_i^S \right| \} > 0$.
\end{claim}
This claim is a direct consequence of Lemma~\ref{lem:ord} (see the proof of this lemma in Section~\ref{sec:pr_low_bound}).

Consider a set of i.i.d samples $(\bm{X}_j, \bm{Y}_j)_{j \in [n]}$, recall the following notation:

\begin{align}
U_{i,j} &:= \bm{X}_{j,i}  \left(\bm{X}_j^{t} \tilde{\beta}^S-\bm{Y}_j \right);\\
\tilde{Z}_{i,n}^S &:= \frac{1}{n} \sum_{j=1}^{n} U_{i,j}\,;\\ 
\tilde{V}_{i,n} &:= \frac{1}{n(n-1)} \sum_{1 \le p < q \le n} \left( U_{i,p}-U_{i,q}\right)^{2} ;\\
\tilde{V}_{i,n}^{+} &:= \max\left(\tilde{V}_{i,n}, \frac{1}{1000} \frac{LM^2}{\rho} \right)  \label{eq:def_vn};\\
\tilde{B} &:= M^2 \|\tilde{\beta}^S\|_1+ M;\\
\text{conf}\left( i,n, \delta \right) &:= \sqrt{\frac{8\tilde{V}_{i,n}^{+}\log(2d n^2/\delta)}{n}} + \frac{28 \tilde B\log(2d n^2/\delta)}{3(n-1)}. \label{eq:def_conf}%, \qquad B=B(M,\tilde{\beta}^S)\geq M \label{eq:def_conf}\\
\end{align}

\paragraph{Proof of Lemma~\ref{lem:tau}.}
For the situation $S=S^*$, the argument is a repetition of the proof of
Lemma~\ref{global-lem} (only considered at the particular selection iteration $k$ where $S_k=S^*$).

We now deal with the situation $S \subsetneq S^*$.
We assume $S$ to be fixed, denote $k=|S|$. As explained in the main body of
the paper, the argument to follow, for fixed $S$, can be transposed directly as a reasoning conditional to $\mathcal{F}_{N_k}$, $N_k$ being the number of data used before starting the $k$-th selection step, with a random $S$ assumed to be $\mathcal{F}_{N_k}$-measurable.

Let $i^* := \text{argmax}_{i \in [d] \setminus S} \{ \left| Z_i^S\right| \}$ (a deterministic quantity). Proceeding by proof via contradiction, suppose that with positive probability, during the execution of \selectproc$(S, \delta_k, 1)$, \tryselectproc either never finishes, or always returns $\tt Success = False$. Assume for the rest of the argument that this event
is satisfied. We can rule out the fact \tryselectproc never stops, since there is
a stopping condition of the type $\mathrm{conf}(i,n,2^{-p}\delta_k) < \mathrm{cst}$,
which is eventually met since $n\rightarrow \infty$ during \tryselectproc, so that
the left-hand side goes to zero and
the right-hand-side constant is positive. Therefore, for all $p \ge 0$
representing the number of recursive calls, \tryselectproc returns
$\tt Success = False$, after having queried a (random) number $n_p$ of data points,
satisfying
(see Algorithms~\ref{algo:step} and~\ref{algo:select_mc}) that

\begin{align}\label{sys:tau_ds}
\left\{
\begin{aligned}
2M \sqrt{\frac{1}{4^p}}   &> \text{conf} \left( i_p, n_p, \frac{\delta_k}{2^{p}}\right);\\
\frac{2}{1-\mu_{S^*}} \text{conf} \left( i^*,n_p-1, \frac{\delta_k}{2^{p}} \right) &> \left| \tilde{Z}_{i^*,n_p-1}^S \right| . \\
\end{aligned}
\right.
\end{align}	

Using the definition of $\text{conf}$ in~\eqref{eq:def_conf}, the first inequality of~\eqref{sys:tau_ds} implies (using the fact that: $\tilde B > M$):
\begin{equation*}
2M\sqrt{\frac{1}{4^p}} > \frac{28M\log \left( 2^{p+1}dn_p^2/\delta_k\right)}{3(n_p-1)}.
\end{equation*}

This implies that $n_p \geq c 2^{p}$ for some factor $c=c(M,\rho,k,d,\delta_k)$, and in particular that
%\ref{eq:lim2} shows particularly that we have:
$\lim\limits_{p \to \infty} n_p = +\infty$.

Now Claim~\ref{claim:bound} shows that $\tilde{V}_{i^*,n}^+$ defined by~\eqref{eq:def_vn} is bounded almost surely by a constant independent of~$p$. Hence, from the definition~\eqref{eq:def_conf}:%using~\ref{eq:lim} we have:
\begin{equation*}
\lim_{p \to \infty} \mathrm{conf}\paren{i^*,n_p-1,\frac{\delta_k}{2^{p+1}}}
% = \lim_{p \to \infty} \left( \sqrt{\frac{8\tilde{V}_{i^*,n_p}^{+}\log(2^{p+1}d n_p^2/\delta_k)}{n_p}} + \frac{28 B\log(2^{p+1}d n_p^2/\delta_k)}{3(n_p-1)} \right)
= 0.
\end{equation*}
We use the second inequality of~\eqref{sys:tau_ds} to conclude that $\lim\limits_{p \to \infty} \abs[1]{\tilde{Z}^S_{i^*, {n_p-1}}} = 0$. By the contradiction hypothesis we assumed that this happens on an event
of positive probability.
On the other hand, since the variables $\tilde{Z}^S_{i^*,n}$ are averages of i.i.d. variables $(\xi_{j})_{1\leq j \leq n}$, and $n_p$
is a stopping time that is lower bounded by $c2^{p}$, Lemma~\ref{lem:martlemma} implies that the variance
of $\tilde{Z}^S_{i^*,n_p}$ goes to 0 as $p$ grows, hence
%	On the other hand, $\lim\limits_{p\to \infty} n_p = +\infty$ also implies that the variables $\left|
$\tilde{Z}^S_{i^*,n_p}$ converges in probability to $ Z_{i^*}^S$.
Finally, we have $\tilde{Z}^S_{i^*,n_p} = \frac{1}{n_p} \xi_p + \frac{n_p-1}{n_p} \tilde{Z}^S_{i^*,n_p-1}$, hence $\abs{\tilde{Z}^S_{i^*,n_p}- \tilde{Z}^S_{i^*,n_p-1}}\leq \frac{2B}{n_p}$, so that $\tilde{Z}^S_{i^*,n_p-1}$ converges in probability to $ Z_{i^*}^S$ as well. 
Therefore $\left|Z_{i^*}^S\right| = 0$, which contradicts the fact that $\max_i \left|Z_i^S\right| > 0$ (see Claim~\ref{lem:evid}).

%
%Now consider the set: $A:= \{ p \ge 0, i_p = i^* \}$. We split the remainder of our analysis in two cases.
%\vspace{2mm}
%
%\textit{Case 1:} the set $A$ is infinite. We can reproduce the same argument as in the data stream
%case by considering the subsequence of strictly increasing integers $(p_q)_{q \ge 0}$ such that: $\forall q \ge 0: i_{p_q} = i^*$.  
%\vspace{2mm}
%
%\textit{Case 2:} the set $A$ is finite. Then $\exists q \ge 0$ such that: $\forall p\ge q: i_p \neq i^*$.
%We deduce that (see Algorithm~\ref{algo:select_db}): $\forall p \ge q$, $\exists n_{i^*,p} \ge 1$ such that:	
%\begin{equation*}
%\left| \tilde{Z}^S_{i_p, n_{i_p,p}-1} \right| + \text{conf} \left( i_p, n_{i_p,p}-1, \frac{\delta_k}{2^p} \right) > \left| \tilde{Z}^S_{i^*, n_{i^*,p}-1}\right| + \text{conf} \left( i_p, n_{i^*,p}-1, \frac{\delta_k}{2^p} \right).
%\end{equation*}
%Hence $	\lim_{p \to \infty} \text{conf} \left( i_p, n_{i^*,p}, \frac{\delta}{2^p} \right) = 0$,
%implying as previously that $ \lim_{p \to \infty} n_{i^*, p} = +\infty$ and contradicting that $A$ is finite.
%
%
%
%

We used the following result:
\begin{lemma}
	\label{lem:martlemma}
	Let $(M_n)_{n\geq 1}$ be a martingale with respect to the filtration $(\mathcal{F}_n)_{n\geq 1}$ and $N$ be a stopping time.
	Let $U_n:= M_n - M_{n-1}$, for $n\geq 1$ (putting $M_0=\e{M_n}$). Assume $\e{U_n^2} \leq A^2$ for all $n\geq 1$,
	and that $N\geq n_0$ a.s. Then:
	\[
	\mathrm{Var}\paren{\frac{M_N}{N}} \leq A^2 \paren{\frac{1}{n_0} + \sum_{i>n_0} i^{-2}}. 
	\]
\end{lemma}
\begin{proof}
	Assume without loss of generality that $E[M_n]=0=M_0$.
	We have, using the fact that the event $\{N\geq j\} = \{N<j\}^c$
	is $\mathcal{F}_{j-1}$-measurable since $N$ is a stopping time:
	\begin{align*}
	\e{M_N^2} = \e{\frac{1}{N^2} \sum_{i,j=1}^N U_iU_j}
	& = \e{\frac{1}{N^2} \sum_{i,j=1}^\infty U_iU_j \ind{N \geq \max(i,j)}}\\
	& = \e{\frac{1}{N^2} \paren{\sum_{i=1}^\infty U_i^2 \ind{N \geq i} + 2 \sum_{i<j}U_i U_j \ind{N \geq j}}}\\
	& \leq \sum_{i=1}^\infty \max(n_0,i)^{-2} \e{U_i^2} + 2 \sum_{i<j}\e{ \frac{1}{N^2} \ind{N \geq j} U_i \underbrace{\e{U_j|\mathcal{F}_{j-1}}}_{=0} }\\
	&\leq A^2 \sum_{i=1}^\infty \max(n_0,i)^{-2}.
	\end{align*}
\end{proof}

Finally, the set of selected features $U$ is not empty since the condition: $ \left|\tilde{Z}_{\hat{i},n_p}\right| > \frac{2}{1 - \mu} \text{ conf}(\hat{i}, n_p , \frac{\delta_k}{2^{p}})$ implies that the condition: 
$ \left|\tilde{Z}_{\hat{i},n_p}\right| - \text{ conf}(\hat{i},n_p , \frac{\delta_k}{2^{p}}) \ge \mu \left( \left|\tilde{Z}_{\hat{i},n_p}\right| + \text{ conf}(\hat{i}, , \frac{\delta_k}{2^{p}})\right)$ is satisfied. Therefore, $U$ contains at least $\hat{i}$. 

\section{Proof of Theorem~\ref{th:main}}\label{sec:th_proof}

Theorem~\ref{th:main} states that \selectproc$(S,\delta,1)$ is guaranteed to select a feature in $S^*$ with high probability if the support is not totally recovered. This part is directly implied by
Lemma~\ref{global-lem} and the fact that the proposed \optimproc and \tryselectproc subroutines satisfy
the optimization confidence property and the selection property, respectively, as established previously.

More importantly, the theorem gives an upper bound on the cumulative computational complexity of the sub-routines \tryselectproc and \optimproc. 

In what follows, following the same approach as in the rest of the paper, we concentrate on a
specific selection iteration (call to \selectproc) and consider $S \subsetneq S^{*}$ to be fixed.
We start by stating some technical lemmas useful for the proof of this theorem.
% In this section we assume that $S \subsetneq S^{*}$, we omit the superscript $S$ in $\tilde{Z}_{i,n}^S$ and $Z_i^S$, and $S^{*}$ in the subscript of $\mu_{S^*}$.

\subsection{Technical Result}

The following concentration inequality is a simple modification of the inequality presented in \cite{DBLP:conf/colt/MaurerP09}  Theorem 10, which consists in assuming that variables $(U_{j,i})_{j \in [n]}$ defined below belong to $[-B,B]$ instead of $[0,1]$.
\begin{lemma}\label{lem:vnbound}
	Consider a fixed $i\in [d]\setminus S$. Suppose Assumption~\ref{ass:ass2} holds with $\bm{X}$ and $\bm{Y}$ being centred random variables. Consider a set of i.i.d. data points $\left(\bm{X}_j, \bm{Y}_j\right)_{j \in [n]}$. Let $\beta \in \mathbb{R}^d$ such that $\|\beta\|_2 \le \frac{2}{\sqrt{\rho}}$ and $\text{supp}(\beta) \subseteq S$. 
	
	Define for a sample $(\bm{X}_j, \bm{Y}_j)$: $U_{j,i} = \left|\bm{X}_{j,i}(\bm{X}_j^{t}\beta -\bm{Y}_j) \right|$, where $\bm{X}_{j,i}$ is the $i^{th}$ feature of $\bm{X}_j$. Finally we define $\tilde{V}_{i,n}$ as: 
	\begin{equation}\label{eq:def_v}
	\tilde{V}_{i,n} = \frac{1}{n(n-1)} \sum_{1 \le l < j \le n} \left(U_{j,i} - U_{l,i} \right)^{2}.
	\end{equation}
	We have in the samples $(\bm{X}_j, \bm{Y}_j)_{j \in [n]}$:
	\begin{align*}
	\mathbb{P} \left( \sqrt{\mathbb{E}\tilde{V}_{i,n}} > \sqrt{\tilde{V}_{i,n} }+ B\sqrt{ \frac{2 \log(1/\delta)}{n-1}}   \right) &\le \delta; \\
	\mathbb{P} \left( \sqrt{\tilde{V}_{i,n}} > \sqrt{\mathbb{E}\tilde{V}_{i,n}} + B\sqrt{ \frac{2 \log(1/\delta)}{n-1}}   \right) &\le \delta, \\
	\end{align*}		
	where $B= M+2 \sqrt{\frac{k}{\rho}} M^2$.
\end{lemma}

We refer to \cite{DBLP:conf/colt/MaurerP09}  Theorem 10, for a proof; recall that Claim~\ref{claim:bound} shows that $\left|U_{j,i} \right| <B$ almost surely.   
\begin{claim}\label{cl:evn}
	Let $i \in [d]\setminus S$. 
	Under the same assumptions as in Lemma~\ref{lem:vnbound}, 
	we have:
	\begin{equation*}
	\mathbb{E}\tilde{V}_{i,n} \le 20\frac{LM^2}{\rho},
	\end{equation*}
	where the expectation is taken with respect to the sample $(\bm{X}_j,\bm{Y}_j)_{j \in [n]}$.
	
	%	Moreover, we have almost surely:		
	%	\begin{equation}\label{eq:vn_bound}
	%	\tilde{V}_{i,n} \le 2B^2.
	%	\end{equation}
\end{claim}

\begin{proof}
	We have by a simple calculation:
	\begin{equation}\label{eq:prclaim}
	\tilde{V}_{i,n} \le \frac{2}{n} \sum_{j=1}^{n} U_{j,i}^{2}.
	\end{equation}
	Hence: 
	\begin{align*}
	\mathbb{E}\brac[1]{\tilde{V}_{i,n}} &\le 2\mathbb{E}_{(x,y)}[U_{1,i}^2]\\
	&\le 2M^2 \mathbb{E}_{(x,y)}[(x^{t}\beta-y)^{2}]\\
	&\le 4M^2 \mathbb{E}_{(x,y)} \left[ \left( x^t\beta\right)^2 + y^2 \right]\\
	&\le 4M^2 \left(\beta^t \Sigma \beta +1\right)\\
	&\le 4M^2 \left(L \norm{\beta}^2 +1\right)\\
	&\le 4M^2 \left(\frac{4L}{\rho} +1\right)\\
	&\le 20  \frac{LM^2}{\rho},
	\end{align*}
	where we used the assumption that $\|\beta\|_2 \le \frac{2}{\sqrt{\rho}}$ (Lemma~\ref{lem:tech}). 
	
	%To prove the second part of the claim, we use the upper bound~\ref{eq:prclaim} along with claim~\ref{claim:bound}.
\end{proof}

\begin{claim}\label{cl:calpure}
	Let $x \ge 1, c \in (0,1)$ and $y >0$ such that:
	\begin{equation}
	\label{eq:hypineq}
	\frac{\log(x/c)}{x} > y.
	\end{equation}
	Then:
	\begin{equation*}
	x < \frac{ 2\log\left( \frac{1}{cy} \right) }{ y}.
	\end{equation*}
	
\end{claim}

\begin{proof}
	%Since $x/c >x \geq 3$, it holds $\log(x/c) \geq 1$.
	Inequality~\eqref{eq:hypineq}
	implies
	\[
	x < \frac{\log(x/c)}{y},
	\]
	and further
	\[
	\log(x/c) < \log(1/yc) + \log \log(x/c) \leq \log(1/yc) + \frac{1}{2}\log(x/c),            
	\]
	since it can be easily checked that $\log(t) \leq t/2$ for all $t>0$.
	Solving and plugging back into the previous display leads to the claim.
\end{proof}

\subsection{Proof of Theorem~\ref{th:main}}

It has already been established based on Lemma~\ref{global-lem} that under Assumptions~\ref{ass:ass4},\ref{ass:ass3}, \ref{ass:ass1} and \ref{ass:ass2}, the set of features $U$ selected by $\selectproc(S, \delta, 1)$ belongs to $S^*$ with high probability, and based on Lemma~\ref{lem:tau} that $U \neq \emptyset$.We therefore now focus on the control of the computational complexity.

Let  $S \subsetneq S^{*}$ be a fixed subset and denote $k:=|S|$. Recall that running $\selectproc(S,\delta,1)$ results in executing \optimproc~and \tryselectproc~alternatively until a condition is verified, implying that at least one feature was selected (see Algorithm~\ref{algo:step}). We use the same notations as in Section~\ref{se:complexitytheory} to denote the computational complexities of \selectproc, \tryselectproc and \optimproc.

Lemma~\ref{lem:tau} shows that, unless interrupted, $\selectproc(S,\delta,1)$ terminates in finite time. Therefore, the number of calls to \optimproc~and \tryselectproc is finite. Let $p$ denote this (random) number.

Let us adopt the following additional notations: For $q\in [p]$, let $m^{(q)}$ denote the number of samples queried during the $q^{\text{th}}$ execution of \optimproc. Let, for $i\in [d]\setminus S$, $n_i^{(q)}$ denote the sample size used to compute $\tilde{Z}_{i}^S$ in the $q^{\text{th}}$ execution of \tryselectproc.

The following lemma provides upper bounds for $C_{\optimproc}$ and $C_{\tryselectproc}$.
\begin{lemma}\label{lem:upc}
	Suppose Assumptions~\ref{ass:ass1} and~\ref{ass:ass2} hold. Let $S\subsetneq S^*$, we have almost surely:
	\begin{enumerate}%[label={\upshape(\roman*)}, align=left, widest=iii, leftmargin=1cm]
		%\item $|\epsilon| \le 2M$    (a.s).
		\item $C_{\optimproc}  \lesssim  \sum_{q=1}^{p} m^{(q)}k$ 
		\item $C_{\tryselectproc} \lesssim \sum_{q=1}^{p} \sum_{i\in [d]\setminus S} n_i^{(q)} $,
	\end{enumerate}
	where $\lesssim$ indicates inequality up to a numerical constant.
\end{lemma}

\begin{proof}
	\begin{enumerate}%[label={\upshape(\roman*)}, align=left, widest=iii, leftmargin=1cm]
		\item %Recall that every time the procedure \tryselectproc returns $\tt Success = False$, \optimproc is run with twice the sample size used in its last run. 
		\optimproc was instantiated using the averaged stochastic gradient descent (Algorithm~\ref{alg:optim}), hence the computational complexity of the $q^{th}$ call of \optimproc is upper bounded by $|S|m^{(q)}$ (up to a numerical constant). Therefore:  
		\begin{equation*}
		C_{\optimproc} \lesssim \sum_{q=1}^{p} m^{(q)}k. 
		\end{equation*}
		
		\item Consider the procedure \tryselectproc given in Algorithm~\ref{algo:select_mc}. In one iteration, calling $\textbf{query-new}(L)$ costs $\mathcal{O}(\left|L\right|)$. Once a sample $(X,Y)$ is obtained, computing the residual $Y-X_{S}^{t}\tilde{\beta}$ costs $|S|$ and updating $\tilde{Z}, v_i$ and $\text{conf}(i)$ for all $i \in L$ costs $\mtc{O}(\left|L\right|)$. Finally, selecting the feature $i^*$ with the maximum $ \{ {| \tilde{Z}_{i}|}+{\text{conf}(i)}\}_{i\in L}$ costs $\mtc{O}(\left|L\right|)$. The cost of the last two tests is $\mtc{O}(\left|L\right|)$. Let $L_{q,t}$ denote the active set of features for the $t$-th iteration
		of \tryselectproc during its $q$-th call. We therefore have 
		\begin{equation*}
		C_{\tryselectproc} \lesssim  \sum_{q=1}^{p} \sum_{t=1}^\infty \abs{L_{q,t}} = \sum_{q=1}^{p} \sum_{i \in [d]\setminus S}
		\sum_{t=1}^\infty \ind{ i \in L_{q,t}}  =
		\sum_{q=1}^{p} \sum_{i \in [d]\setminus S} n_i^{(q)}. 
		\end{equation*}

	\end{enumerate}

\end{proof}

In order to provide a control on the computational complexity of $C_{\selectproc}$, we need to derive a control on the (random) quantities $p$, $m^{(q)}$ and $n_i^{(q)}$ for $1\le q \le p$ and $i \in [d] \setminus S$. 
In the remainder of this proof, $\kappa$ will refer to a constant depending only on $L,\rho$ and $M$. The value of $\kappa$ may change from line to line.

Recall the definition:	
\begin{equation}\label{eq:defconf}
\text{conf} \left(i, n, \delta\right) := \sqrt{\frac{8\tilde{V}_{i,n}^+ \log\left(2dn^2/\delta\right)}{n}} + \frac{28\tilde{B}\log\left(2dn^2/\delta\right)}{3(n-1)},
\end{equation}	
where $\tilde{B}:= M+ M^2 \|\tilde{\beta}^S\|_1  $  and $\tilde{V}_{i,n}^+$ is given by~\eqref{eq:def_v}. Since $\text{ conf}(.)$ is a data-dependent function, the claim below provides a deterministic upper bound.

\begin{claim}\label{cl:bconf}
	Suppose Assumption~\ref{ass:ass2} holds with $X$ and $Y$ being centered random variables. Let $B_k :=  M+ 2M^2 \sqrt{\frac{k}{\rho}}$ and define:
	\begin{equation}\label{eq:defbconf}
	\overline{\mathrm{conf}} \left( n, \delta\right) := 8\sqrt{\frac{LM^2 \log\left(2dn^2/\delta\right)}{\rho n}} + \frac{27B_k\log\left(2dn^2/\delta\right)}{n}.	
	\end{equation}
	
	Then, for all $\delta \in (0,1)$, with probability at least $1 - \delta$ we have: $\forall i \in [d] \setminus S$, $\forall n \ge 2$:
	\begin{equation*}
	\overline{\mathrm{conf}}(n,\delta) \geq \mathrm{conf} \left( i,n, \delta\right).
	\end{equation*}
	
\end{claim}

\begin{proof}
	Let $\delta \in (0,1)$. Lemma~\ref{lem:vnbound} and Claim~\ref{cl:evn} show that with probability at least $1-\delta$, $\forall i \in [d]\setminus S$, $n\ge 2$:
	\begin{equation*}
	\sqrt{\tilde{V}_{i,n}} \le \sqrt{8\frac{LM^2}{\rho}} + B_k \sqrt{\frac{2 \log\left(2dn^2/\delta\right)}{n-1}}.
	\end{equation*}
	Moreover, recall that: $\tilde{B} = M^2 \norm[1]{\tilde{\beta}^S}_1 + M$. Since $\tilde{\beta}^S \in \mathcal{B}_k\left(0, \frac{2}{\sqrt{\rho}}\right)$, we have: $\norm[1]{\tilde{\beta}^S}_1 \le \sqrt{k} \norm[1]{\tilde{\beta}^S}_2 \le 2 \sqrt{\frac{k}{\rho}}$. Hence, we have almost surely: $\tilde{B} \le B_k$.
	Using the bound on $\tilde{B}$ and on $\tilde{V}_{i,n}$ we obtain the conclusion. 
\end{proof}

Let us denote $\delta_k := 1/(2(k+1)(k+2))$.
At each iteration of OOMP (Algorithm~\ref{algo:oomp}), the procedure \selectproc is called with inputs $\left(S, \delta_k, 1\right)$. Then \selectproc is run following Algorithm~\ref{algo:step} recursively until a condition, implying that at least an additional feature was selected, is verified. Thus, the inputs of the $q^{th}$ call to \selectproc are $\left(S, \delta_k/2^q, 1/4^q\right)$.

\paragraph{Computational complexity bounds:}\mbox{}\\

We define the following key quantities: for $q \ge 1$, for $i \in [d]\setminus S$, let:

\begin{equation}\label{eq:wi}
W_i := \max\left \lbrace \frac{\left|Z_{i^*}^S\right| - \left|Z_{i}^S\right|}{4}; \frac{1 - \mu}{3 - \mu} \left| Z_{i}^S\right| \right \rbrace,
\end{equation}

and 
\begin{equation}\label{eq:nq}
\bar{n}_{i}^{(q)} := \min \left\lbrace n >0: \overline{\text{conf}} \left(n, 2^{-q}{\delta_k}\right) < W_i \right \rbrace,
\end{equation}

where $i^* \in \text{argmax}_{i\in[d]} \left|Z_i^S\right|$.

The following argument proves the existence of $\bar{n}_i^{(q)}$: By assumption $S\subsetneq S^*$, Claim~\ref{lem:evid} shows that $\left| Z_{i^*}^S\right| > 0$, thus $W_1>0$ as well. Definition~\ref{eq:defbconf} shows that $\overline{\text{conf}}(.,\delta)$ is strictly decreasing and converges to $0$ when $n\to \infty$, which guarantees that $\bar{n}_i^{(q)}$ exists.

The technical result below gives an upper bound for $\bar{n}_i^{q}$:

\begin{lemma}\label{le:calcnq}
	Let $i \in [d]\setminus S$ and $\bar{n}_{i}^{(q)}$ be defined by~\eqref{eq:nq}. Let $W_i$ be the quantity defined by \eqref{eq:wi},
	We have:
	\begin{equation*}
	\bar{n}_i^{(q)} \le \kappa \max\left\lbrace \frac{1}{ W_{i}^2} , \frac{\sqrt{k}}{ W_{i}}
	\right\rbrace \log \left( \frac{B_kd2^q}{\delta_k W_{i}}\right),
	\end{equation*}
	
	where $\kappa$ depends only on $L$, $M$ and $\rho$, and $B_k :=  M+ 2M^2 \sqrt{\frac{k}{\rho}}$.

\end{lemma}

\begin{proof}
	By definition of $\bar{n}_i^{(q)}$ we have:
	\begin{equation*}
	\overline{\text{conf}} \left(\bar{n}_i^{(q)}-1, 2^{-q}{\delta_k}\right) \ge W_i.
	\end{equation*}	
	Using Definition~\ref{eq:defbconf} we have:
	\begin{equation*}
	8\sqrt{\frac{LM^2 \log\left(2d(\bar{n}_i^{(q)}-1)^2 2^q/\delta_k\right)}{\rho \left( \bar{n}_i^{(q)} -1\right)}} + \frac{27B_k\log\left(2d(\bar{n}_i^{(q)}-1)^2 2^q/\delta_k\right)}{\bar{n}_i^{(q)}-1} \ge W_i.
	\end{equation*}
	Now, using the fact that $a+b >c \implies \max\{a,b\}> c/2$:
	\begin{align}\label{eq:n/logn}
	\left\{
	\begin{aligned}
	\frac{ \log\left(2d(\bar{n}_i^{(q)}-1)2^q/\delta_k\right)}{\bar{n}_i^{(q)}-1}   &\ge \frac{\rho}{256 LM^2} W_i^2 \\
	&\text{or}\\
	\frac{\log\left(2d(\bar{n}_i^{(q)}-1)2^q/\delta_k\right)}{\bar{n}_i^{(q)}-1} &\ge \frac{1}{54B_k} W_i.\\
	\end{aligned}
	\right.
	\end{align}	
	Now we use Claim~\ref{cl:calpure}:
	\begin{align*}
	\left\{
	\begin{aligned}
	\bar{n}_i^{(q)}-1   &\le \frac{512LM^2}{\rho W_i^2} \log \left( \frac{128LM^2d2^q}{\rho \delta_k W_i^2}\right) \\
	&\text{or}\\
	\bar{n}_i^{(q)}-1 &\le \frac{108B_k}{W_i} \log \left( \frac{27B_kd2^q}{\delta_k W_i}\right). \\
	\end{aligned}
	\right.
	\end{align*}
	Finally, we upper bound $\bar{n}_i^{(q)}$ by the maximum of these bounds.
\end{proof}

For the rest of the proof, we upper bound the complexities of \tryselectproc and \optimproc using $\bar{n}_i^{(q)}$. The lemma below relates the quantities $n_i^{(q)}$ and $\bar{n}_i^{(q)}$.

\begin{lemma}\label{lem:boundnq}
	Under the assumptions of Theorem~\ref{th:main}: 
	\begin{equation*}
	\mathbb{P} \left( \forall q \le p, \forall i \in [d]\setminus S: n_i^{(q)} \le \bar{n}_i^{(q)} +1 \right) \ge 1-3\delta_k.
	\end{equation*}
\end{lemma}

\begin{proof}
	% First recall that $n_i^{(q)}$ is the number of samples used to compute $\tilde{Z}_{i,n_i^{(q)}}$, hence at iteration $n_i^{(q)}-1$, the feature $i$ was not eliminated. Using the notations of Algorithms~\ref{algo:select_mc}: $\forall i, j \in L$: $n_i^{(q)} = n_j^{(q)}$, since all the non-eliminated features were computed the same number of data-points.
	
	% By definition of $n_i^{(q)}$  we have:
	Let us fix $i \in [d]\setminus S$ and $q\in [p]$. We consider the iteration $n = n_i^{(q)}-1$ during the $q$-th call of \tryselectproc, and let $L$ denote the active set of features for this iteration.
	
	Let $\hat{i} \in \text{argmax}_{j \in L} \left\lbrace\left|\tilde{Z}_{j, n}\right| + \text{ conf}(j,n, {\delta_k}{2^{-q}} )\right\rbrace$. We have by design of Algorithm~\ref{algo:select_mc} (since $n < n_i^{(q)}$):
	\begin{equation*}
	\frac{2}{1-\mu} \text{conf} \left( \hat{i}, n, 2^{-q}{\delta_k} \right) > \left| \tilde{Z}_{\hat{i}, n}^S\right|,
	\end{equation*}
	hence:
	\begin{equation*}
	\frac{3-\mu}{1-\mu} \text{conf} \left( \hat{i}, n, 2^{-q}{\delta_k} \right) > \left| \tilde{Z}_{\hat{i}, n}^S\right|+\text{conf} \left( \hat{i}, n, 2^{-q}{\delta_k} \right).
	\end{equation*}
	We therefore have (by definition of $\hat{i}$):
	\begin{equation}\label{eq:bound}
	\frac{3-\mu}{1-\mu} \text{conf} \left( \hat{i},n, 2^{-q}{\delta_k} \right) > \left| \tilde{Z}_{i, n}^S\right|+\text{conf} \left( i, n, 2^{-q}{\delta_k} \right).
	\end{equation}
	%	Let us fix for now an iteration $q$ of the recursive call to \selectproc.
	As in the proof of Lemma~\ref{global-lem}, let us denote $B_{k,q}$
	the event ``the $q$-th call to \optimproc  during
	the $k$-th selection iteration, if it took place, returned $\tilde{\beta}^S$ such that~\eqref{eq:condopt} holds'' and recall that the
	optimization confidence property guarantees $\prob{B_{k,q}^c}\leq \delta_k 2^{-q}$. 
	Provided this control holds, recall that Proposition~\ref{prop:conczi} shows that
	\begin{equation} \label{eq:probprop4}
	\mathbb{P} \left(\forall m\geq 2, \forall j \in [d], \left| \tilde{Z}_{j, m}^S - Z_j^S\right|\leq \frac{1}{2}\text{conf} \left( j, m, 2^{-q}{\delta_k} \right) + M  2^{-q}  \Big| B_{k,q} \right) \ge 1-\delta_k2^{-q}.
	\end{equation}
	Let us denote by $A_{k,q}$ the event:
	\begin{equation}\label{eq:eventA}
	\forall m\ge 2, \forall j \in [d]\setminus S: \quad \left| \tilde{Z}_{j, m}^S - Z_j^S\right| \leq \text{conf} \left( j, m, 2^{-q}{\delta_k} \right)
	\end{equation}	
	Recall that at iteration $n$, we must have:
	\begin{equation*}
	\forall i \in [d]\setminus S: \qquad  \mathrm{conf}(i,n,2^{-q}\delta_k) \geq 2M2^{-q},
	\end{equation*}
	thus~\eqref{eq:probprop4} implies
	\begin{equation}\label{eq:probbis}
	\mathbb{P} \left( A_{k,q} \Big| B_{k,q}\right) \ge 1-\delta_k2^{-q},
	\end{equation}
	Using~\eqref{eq:bound}, we have:
	\begin{equation}\label{eq:boundhati}
	\mathbb{P} \left(\frac{3-\mu}{1-\mu} \text{conf} \left( \hat{i}, n, 2^{-q}{\delta_k} \right) > \left| Z_{i}^S \right|
	\Big| B_{k,q} \right) \ge 1-\delta_k 2^{-q}.
	\end{equation}
	Using Claim~\ref{cl:bconf}, it holds:
	\begin{equation}\label{eq:claim6}
	\mathbb{P} \left( \forall m \ge 2, \forall i \in [d]\setminus S: \overline{\text{conf}} \left( m, {\delta_k}{2^{-q}}\right) > \text{conf} \left( i, m, {\delta_k}{2^{-q}} \right) \right) \ge 1 - \delta_k 2^{-q},
	\end{equation}
	therefore, \eqref{eq:boundhati} gives:
	\begin{equation}\label{eq:boundf1}
	\mathbb{P} \left( \overline{\text{conf}} \left(n, 2^{-q}{\delta_k} \right) > \frac{1-\mu}{3-\mu}\left| Z_{i}^S \right|
	\Big| B_{k,q} \right) \ge 1-\delta_k 2^{-q}.
	\end{equation}
	Let $i^* \in \text{argmax}_{j \in [d]\setminus S} \left|Z_j^S\right|$. Suppose that event $A_{k,q}$ is true. Let us show that $i^* \in  L$. In fact, if $i^* \notin L$, we have by design of the procedure \tryselectproc: $\exists m <n$ and $\exists j \in [d]\setminus S$ such that:
	\begin{equation*}
	\left|\tilde{Z}^S_{i^*, m}\right| + \text{ conf}(i^*, m, {\delta_k}{2^{-q}}) < \left|\tilde{Z}^S_{j, m}\right| - \text{ conf}(j, m, {\delta_k}{2^{-q}})
	\end{equation*}
	By definition of event $A_{k,q}$ in \eqref{eq:eventA}. We conclude that:
	\begin{equation*}
	\left|Z_{i^*}^S\right| < \left| Z^S_j\right|,
	\end{equation*}
	which contradicts the definition of $i^*$. We therefore have: if $A_{k,q}$ is true then $i^* \in L$.
	
	Moreover, by design of \tryselectproc:
	\begin{align*}
	\left| \tilde{Z}_{i, n}^S\right| + \text{conf} \left( i, n, {\delta_k}{2^{-q}} \right) &\ge  \left| \tilde{Z}_{\hat{i}, n}^S \right| - \text{conf} \left( \hat{i}, n, {\delta_k}{2^{-q}} \right)  \\
	&=  \left| \tilde{Z}_{\hat{i}, n}^S \right|+ \text{conf} \left( \hat{i}, n, {\delta_k}{2^{-q}} \right) - 2\text{conf} \left( \hat{i}, n, {\delta_k}{2^{-q}} \right)  \\
	&\ge \left| \tilde{Z}_{i^*, n}^S \right|+ \text{conf} \left( i^*, n, {\delta_k}{2^{-q}} \right)- 2\text{conf} \left( \hat{i}, n, {\delta_k}{2^{-q}} \right)  
	\end{align*}	
	%	Let $i^* \in \text{argmax}_{j \in [d]\setminus S} \{ \left|Z_j^S\right|\}$. We have: 
	%	\begin{equation*}
	%	\mathbb{P} \left( i^* \in L \Big| B_{k,q} \right) \ge 1-\delta_k 2^{-q}.
	%	\end{equation*}
	Therefore:
	\begin{equation*}
	\left| \tilde{Z}_{i, n}^S\right| - \text{conf} \left( i, n, {\delta_k}{2^{-q}} \right) + 2\text{conf} \left( i, n, {\delta_k}{2^{-q}} \right) \ge \left| \tilde{Z}_{i^*, n}^S \right|+ \text{conf} \left( i^*, n, {\delta_k}{2^{-q}} \right)- 2\text{conf} \left( \hat{i}, n, {\delta_k}{2^{-q}} \right).
	\end{equation*}	
	Since event $A_{k,q}$ is true, we upper bound the quantity : $\left| \tilde{Z}_{i, n}^S\right| - \text{conf} \left( i, n, {\delta_k}{2^{-q}} \right) $, and lower bound the quantity: $\left| \tilde{Z}_{i^*, n}^S \right|+ \text{conf} \left( i^*, n, {\delta_k}{2^{-q}} \right) $. We obtain:
	
	\begin{equation*}
	\left| Z_{i}^S \right| + 2\text{conf} \left( i, n, {\delta_k}{2^{-q}} \right) \ge \left| Z_{i^*}^S \right| - 2\text{conf} \left( \hat{i}, n, {\delta_k}{2^{-q}} \right).
	\end{equation*}
	As a conclusion, we have:
	\begin{equation*}
	\mathbb{P} \left(\left| Z_{i}^S \right| + 2\text{conf} \left( i, n, {\delta_k}{2^{-q}} \right) \ge \left| Z_{i^*}^S \right| - 2\text{conf} \left( \hat{i}, n, {\delta_k}{2^{-q}} \right)
	\Big| B_{k,q} \right) \ge 1-\delta_k 2^{-q},
	\end{equation*}	
	which leads to: 
	\begin{equation*}
	\mathbb{P} \left( 2\text{conf} \left( i, n, {\delta_k}{2^{-q}} \right) + 2\text{conf} \left( \hat{i}, n, {\delta_k}{2^{-q}} \right) \ge \left| Z_{i^*}^S \right| - \left| Z_{i}^S \right|
	\Big| B_{k,q} \right) \ge 1-\delta_k 2^{-q}.
	\end{equation*}
	Finally, we use \eqref{eq:claim6} to upper bound $\text{conf} \left( i, ., . \right)$ and $\text{conf} \left( \hat{i}, ., . \right)$ using $\overline{\text{conf}}(.)$:
	\begin{equation}\label{eq:boundf2}
	\mathbb{P} \left( 4\overline{\text{conf}} \left( n, {\delta_k}{2^{-q}} \right) \ge \left| Z_{i^*}^S \right| - \left| Z_{i}^S \right|
	\Big| B_{k,q} \right) \ge 1-\delta_k 2^{-q}.
	\end{equation}	
	We obtain, using \eqref{eq:boundf2} and \eqref{eq:boundf1}:
	\begin{equation}\label{eq:boundf3}
	\mathbb{P} \left( \overline{\text{conf}} \left( n, {\delta_k}{2^{-q}} \right) \ge W_i
	\Big| B_{k,q} \right) \ge 1-\delta_k 2^{-q};
	\end{equation}	
	furthermore by definition of $\bar{n}_i^{(q)}$ (see~\eqref{eq:nq}):
	\begin{equation}\label{eq:boundf4}
	\overline{\text{conf}} \left( \bar{n}_i^{(q)}, {\delta_k}{2^{-q}}\right) \le W_i.
	\end{equation}
	Using inequalities \eqref{eq:boundf3}-\eqref{eq:boundf4}, we have:
	
	\begin{equation*}
	\mathbb{P} \left( \overline{\text{conf}} \left( n^{(q)}_i - 1, {\delta_k}{2^{-q}}\right) \ge \overline{\text{conf}} \left( \bar{n}_i^{(q)}, {\delta_k}{2^{-q}}\right)\Big| B_{k,q} \right) \ge 1-2\delta_k 2^{-q}.
	\end{equation*}
	Denoting $D_{k,q}$ the event appearing above, we use
	$\prob{D_{k,q}^c} \leq \prob{D_{k,q}^c \cap B_{k,q}}  +
	\prob{B_{k,q}^c} \leq \prob{D_{k,q}^c | B_{k,q}}  +
	\prob{B_{k,q}^c} \leq 2\delta_k2^{-q}$ together with a union bound
	over $q\geq 1$ to get
	\[
	\mathbb{P} \left( \forall q \leq p: \overline{\text{conf}} \left( n_i^{(q)}-1, {\delta_k}{2^{-q}}\right) \ge \overline{\text{conf}} \left( \bar{n}_i^{(q)}, {\delta_k}{2^{-q}}\right)\right) \ge 1-3\delta_k.
	\]
	The result follows from the fact that the function $n \to \overline{\text{conf}}(n,\delta)$ is decreasing for all $\delta \in (0,1)$. 
\end{proof} 

In order to get an upper bound for the computational complexity of \selectproc, we now develop a
high probability bound on $p$ (the total number of calls of \tryselectproc and \optimproc during one call of $\selectproc\left(S, \delta_k, 1\right)$).
%This is what we do in the next lemma.
%The lemma below gives an upper bound for $p$:

\begin{lemma}\label{lem:boundp}
	Suppose $p\ge 2$. Under the assumptions of Theorem~\ref{th:main}, $p$ satisfies the following inequality: 
	\begin{equation*}
	\mathbb{P} \left( 2^p \leq \kappa \max\left\lbrace \frac{1}{W_{i^*}}; \sqrt{\frac{B_k}{W_{i^*}}} \right \rbrace\right) \ge 1-3\delta_k,
	\end{equation*}
	where $\kappa$ only depends on $(\rho,L,M)$.
\end{lemma}

\begin{proof}
	By definition of $p$, the procedure \tryselectproc returns $\tt Success = False$ in its call number $p-1$. Then (see Algorithm~\ref{algo:select_mc}) $\exists i \in [d]\setminus S$ such that:
	\begin{equation*}
	2M\sqrt{\frac{1}{4^{p-2}}}  > \text{conf}\left( i, n_i^{(p-1)}, \frac{\delta_k}{2^{p-2}}\right).
	\end{equation*}
	Using Definition~\ref{eq:defconf} for $\text{conf}$, we deduce:	
	\begin{equation*}
	2M \sqrt{\frac{1}{4^{p-2}}}  > \sqrt{\frac{8 \tilde{V}_{i,n_i^{(p-1)}}^+ \log\left( 2^{p-1}d(n_i^{(p-1)}-1)^2/\delta_k\right)}{n_i^{(p-1)}-1}}.
	\end{equation*}
	Recall that by definition of $\tilde{V}_{i,n_i}^+$, it holds 
	\begin{equation*}
	\tilde{V}_{i,n_i}^{+} \ge \frac{1}{10^3} \frac{LM^2}{\rho},
	\end{equation*}
	therefore	
	\begin{equation*}
	2M \frac{1}{2^{p-2}}  > \frac{1}{11}\sqrt{ \frac{LM^2}{\rho \left(n_i^{(p-1)}-1\right)} \log\left( 2^{p-1} d\left(n_i^{(p-1)}- 1\right)^2/\delta_k\right)},
	\end{equation*}
	and finally	
	\begin{equation*}
	2^{p}  \leq c \sqrt{\frac{\rho( n_i^{(p-1)}-1)}{L\log\left( 2^pd(n_i^{(p-1)}-1)/\delta_k\right)}},
	\end{equation*}
	for $c$ an absolute numerical constant.
	
	Using Lemma~\ref{lem:boundnq} along with the fact that the function $n\to n/\log(an)$ is non-decreasing for $a>1$, we have:
	\begin{equation*}
	\mathbb{P} \left( 2^p \leq c \sqrt{\frac{\rho \bar{n}_i^{(p-1)}}{L\log\left( 2^pd\bar{n}_i^{(p-1)}/\delta_k\right)}}\right) \ge 1-3\delta_k.
	\end{equation*}
	Recall from~\eqref{eq:n/logn} that there is a numerical constant $c'$ such that:
	\begin{equation*}
	\frac{ \log\left(d(\bar{n}_i^{(p-1)}-1)2^q/\delta_k\right)}{\bar{n}_i^{(p-1)}-1}   \ge c' \max\left\lbrace \frac{\rho}{LM^2} W_i^2; \frac{1}{B_k} W_i \right\rbrace .
	\end{equation*}	
	Finally, it is elementary to check that $\forall x \in [0, \left|Z_{i^*}^S\right|]$:
	\begin{align*}
	\max\left\lbrace \frac{1}{4}\left( \left|Z_{i^*}^S\right| - x \right), \frac{1-\mu}{3 - \mu} x \right \rbrace &\ge \frac{3-\mu}{7 - 5\mu} \left|Z_{i^*}^S\right| \\
	&\ge \frac{2}{7} W_{i^*}.
	\end{align*}
	Hence, taking $x=\abs{Z_i^S}$ above, we get $W_i \ge \frac{2}{7} W_{i^*}$. As a conclusion, there exists a constant $\kappa$ depending only on $\rho, L$ and $M$ such that:
	\begin{equation*}
	\mathbb{P} \left( 2^p \leq \kappa \max\left\lbrace \frac{1}{W_{i^*}}; \sqrt{\frac{B_k}{W_{i^*}}} \right \rbrace\right) \ge 1-3\delta_k.
	\end{equation*}

\end{proof}

Recall that we have: $C_{\tryselectproc} \lesssim \sum_{q=1}^{p} \sum_{i \in [d]\setminus S}     n_i^{(q)}$ (Lemma~\ref{lem:upc}). Therefore, using Lemmas~\ref{le:calcnq}, \ref{lem:boundnq} and~\ref{lem:boundp} above, we have with probability at least $1 - 3 \delta_k$:
\begin{align*}
C_{\tryselectproc} &\lesssim \sum_{q=1}^{p} \sum_{i\in [d]\setminus S} n_i^{(q)}\\ 
&\lesssim \sum_{q=1}^{p} \sum_{i\in [d]\setminus S} \bar{n}_i^{(q)} \\
&\le \sum_{q=1}^{p} \sum_{i\in [d]\setminus S} \kappa \max\left\lbrace \frac{1}{W_i^2} , \frac{\sqrt{k}}{W_i}
%\frac{B}{\left( 1 - \mu_{S^{*}} \right) \left|Z^S_{i^*}\right|}
\right\rbrace \log \left( \frac{B_kd2^q}{\delta_k W_i}\right)\\
&\le p\kappa \sum_{i\in [d]\setminus S} \max\left\lbrace \frac{1}{W_i^2} , \frac{\sqrt{k}}{W_i}
%\frac{B}{\left( 1 - \mu_{S^{*}} \right) \left|Z^S_{i^*}\right|}
\right\rbrace \log \left( \frac{B_kd2^{p}}{\delta_k W_i}\right).
\end{align*}
In particular, Lemma~\ref{lem:boundp} shows that:
\begin{equation*}
\mathbb{P} \left( 2^p \lesssim  \max\left\lbrace \frac{1}{W_{i^*}}; \sqrt{\frac{B_k}{W_{i^*}}} \right \rbrace\right) \ge 1-3\delta_k.
\end{equation*}
Hence, with probability at least $1 - 3\delta_k$:
\begin{equation*}
\log(2^p) \le \kappa \log\left( \frac{k}{W_{i^*}}\right) .
\end{equation*}
We conclude after some elementary bounding that, with probability at least $1-6\delta_k$:
\begin{equation*}
C_{\tryselectproc}  \leq \kappa  \sum_{i \in [d]\setminus S} \max\left\lbrace \frac{1}{W_i^2} ; \frac{\sqrt{k}}{W_i} \right \rbrace \log \left( \frac{d}{ \delta_k W_{i^*}}\right)\log \left( \frac{k}{ W_{i^*}}\right),
\end{equation*}
where $\kappa$ is a constant depending only on $L,\rho$ and $M$.

Moreover, since the inputs of \optimproc at its $q^{th}$ call when executing $\selectproc\left(S, \delta_k, 1\right)$ are: $\left(S, \delta_k/ 2^{q}, 1/4^{q}\right)$. Hence, (by design of Algorithm~\ref{alg:optim}) we have:
%\begin{equation} \label{eq:mqds}
%m^{q,ds} \le \kappa \max\left\lbrace \sqrt{k2^q}, 2^q \right\rbrace \log\left(2^q/\delta_k\right) \log\left(2^q\right),
%\end{equation}  
\begin{equation} \label{eq:mqds}
m^{(q)} \le \kappa  k^2 4^q \log\left( \frac{2^q}{\delta_k}\right),
\end{equation} 
where $\kappa$ depends on $L$, $M$, and $\rho$. We therefore have:
\begin{align*}
C_{\optimproc} &\lesssim \sum_{q=1}^{p} k m^{(q)} \\
&\le \sum_{q=1}^{p} \kappa  k^3 2^{2q} \log\left( \frac{2^q}{\delta_k}\right) \\
&\le  \kappa k^3 2^{2(p+1)} \log\left(\frac{2^p}{\delta_k}\right).
\end{align*}
We conclude applying Lemma~\ref{lem:boundp}: with probability at least $1 - 3\delta_k$,
\begin{equation*}
C_{\optimproc} \le \kappa  k^3 \max\left\lbrace \frac{1}{W_{i^*}^2} , \frac{\sqrt{k}}{W_{i^*}} \right\rbrace \log \left( \frac{k}{ \delta_k W_{i^*}}\right),
\end{equation*}
%
%This shows in particular that: $m^{q,ds} \le m^{p,ds}$ for all $q\in[p]$. Therefore: $C_{\optimproc}^{ds} \lesssim pkm^{p,ds}$. To conclude, we use the bound~\ref{eq:mqds} along with lemma~\ref{lem:boundp} to have with probability at least $1 - \delta_k$:
%
%\begin{equation*}
%	C_{\optimproc}^{ds} \lesssim \kappa \quad k \max\left\lbrace \frac{1}{\left( 1 - \mu_{S^{*}} \right)^2 \left|Z^S_{i^*}\right|^2} ; \frac{k^{3/4}}{\sqrt{\left( 1 - \mu_{S^{*}} \right) \left|Z^S_{i^*}\right|}} \right \rbrace \log^2 \left( \frac{d}{\delta_k \left( 1-\mu_{S^{*}}\right) \left| Z_{i^*}^S\right|}\right).
%\end{equation*}
where $\kappa$ is a factor depending only on $L$, $M$ and $\rho$.

\section{Lower bound on the scores $Z_i^S$:}\label{sec:pr_low_bound}

Let us denote $(\beta^{S^*}_{(i)})_i$ the reordered coefficients of $\beta^{S^*}$: $|\beta^{S^*}_{(1)}|\ge...\ge|\beta^{S^*}_{(s^*)}|$. Lemma~\ref{lem:ord} provides a lower bound for $ \max_{i \in [d] \setminus S} \left| Z_i^S\right|$. 

\begin{lemma}\label{lem:ord}
	Suppose Assumptions~\ref{ass:ass4},~\ref{ass:ass3},~\ref{ass:ass1} and~\ref{ass:ass2} hold. Assume that  $S \subsetneq S^{*}$ and denote $k:=|S|$, we have:
	\begin{equation*}
	\max_{i \in [d] \setminus S} \left| Z_i^S\right| \ge \sqrt{\frac{\rho^{3}}{L}} \frac{1}{\sqrt{s^* - k}} \norm{\beta^{S^*} - \beta^S}_2 \ge \sqrt{\frac{\rho^{3}}{L}} \frac{1}{\sqrt{s^* - k}} \norm{\beta^{S^*}_{S^* \setminus S}}_2.                  
	\end{equation*}
	
\end{lemma} 
In this section we prove Lemma~\ref{lem:ord}, we begin by presenting the following technical lemmas adapted from \cite{zhang2009consistency} to fit the random design.

% We consider following assumption:
\begin{claim}\label{cl:add}
	Suppose Assumptions~\ref{ass:ass4} and~\ref{ass:ass1} hold. Then	
	for all $i\in [d]$: $\rho \le \mathbb{E}[x_i^2] \le  L$.
\end{claim}

Claim~\ref{cl:add} is a direct consequence of Assumption~\ref{ass:ass1} stating that the eigenvalues of $\Sigma_S$ are lower bounded by $\rho$ and upper bounded by $L$, and the observation that $\mathbb{E}\left[x_i^2\right]$ are the diagonal terms of $\Sigma_S$.

\begin{lemma}\label{lem:ord1}
	Let $x,y$ and $z$ be real valued bounded and centered random variables, such that $\mathbb{E}\left[x^2\right]=1$. We have:
	\begin{equation*}
	\inf_{\alpha \in \mathbb{R}} \mathbb{E} \left[\left( y+\alpha x-z \right)^2\right] = \mathbb{E} \left[ \left(y-z\right)^2\right] -  \frac{1}{\mathbb{E}\left[x^2\right]} \mathbb{E}\left[ x \left(y-z \right)\right]  ^2. 
	\end{equation*}
\end{lemma} 

\begin{proof}
	The proof follows from simple algebra, the minimum is attained for $\alpha = -\frac{\mathbb{E}\left[ x(y-z)\right]}{\mathbb{E}[x^2]}$. 
\end{proof}

\begin{lemma}\label{lem:ord2}
	Let Assumptions~\ref{ass:ass4}, \ref{ass:ass3}, \ref{ass:ass1} and~\ref{ass:ass2} hold, consider a fixed subset $S \subsetneq S^*$ and denote $k:=|S|$. We have the following:
	\begin{equation*}
	\inf_{\alpha \in \mathbb{R}, i \in S^*\setminus S} \mathbb{E} \left[ \left( x^t \beta^{S} + \alpha \beta^{S^*}_i x_i - y \right)^2\right] \le	\mathbb{E} \left[ \left(x^t\beta^S - y\right)^2\right] - \frac{1}{s^* - k} \frac{\rho}{L} \mathbb{E}\left[ \paren[1]{x^t(\beta^{S^*} - \beta^S)}^2\right].
	\end{equation*}
\end{lemma}

\begin{proof}
	Let $\eta \in \mathbb{R}$, we have:
	\begin{align*}
	\min_{i \in S^*\setminus S} \mathbb{E} \left[ \left( x^t \beta^{S} + \eta \beta^{S^*}_ix_i - y \right)^2\right] & \le \frac{1}{s^* - k} \sum_{i \in S^* \setminus S} \mathbb{E} \left[ \left( x^t \beta^{S} + \eta  \beta^{S^*}_i x_i - y \right)^2\right] \\
	&\le \mathbb{E}\left[ \left(x^t\beta^S - y\right)^2\right] + \frac{1}{s^*-k} \sum_{i \in S^* \setminus S}  \eta^2 \left( \beta^{S^*}_i\right)^2 \mathbb{E}\left[x_i^2\right] \\
	+ &\frac{1}{s^* - k} \sum_{i \in S^* \setminus S} 2\eta\beta^{S^*}_i  \mathbb{E} \left[ x_i \left(x^t\beta^S - y\right)\right].
	\end{align*}
	Recall that optimality of $\beta^S$ implies that for all $i\in S$: $\mathbb{E}\left[x_i \left(x^t\beta^S - y\right)\right]=0$. Hence:
	\begin{align*}
	\sum_{i \in S^* \setminus S}\beta^{S^*}_i  \mathbb{E} \left[ x_i \left(x^t\beta^S - y\right)\right]
	&= \sum_{i \in S^* \setminus S}\left( \beta^{S^*}_i - \beta^{S}_i\right) \mathbb{E} \left[ x_i \left(x^t\beta^S - y\right)\right]\\
	&= \sum_{i \in S^* }\left( \beta^{S^*}_i - \beta^{S}_i\right) \mathbb{E} \left[ x_i \left(x^t\beta^S - y\right)\right] \\
	&= \sum_{i \in S^* }\left( \beta^{S^*}_i - \beta^{S}_i\right) \mathbb{E} \left[ x_i \left(x^t\beta^S - x^t\beta^{S^*}\right)\right] \\
	&= \mathbb{E} \left[ \left( \beta^{S^*} - \beta^S \right)^t x \left(x^t \beta^{S} - x^t \beta^{S^*}\right) \right]\\
	&= \mathbb{E} \left[ \left(x^t \left( \beta^{S^*} - \beta^S \right)\right)^2 \right].
	\end{align*}
	Therefore:
	\begin{multline*}
	\left(s^* - k\right)\min_{i \in S^*\setminus S} \mathbb{E} \left[ \left( x^t \beta^{S} + \eta  \beta^{S^*}_i x_i - y \right)^2\right]\\
	\begin{aligned}
	& \le \left(s^* - k\right) \mathbb{E} \left[ \left(x^t\beta^S -y\right)^2\right]  \\
	& \qquad +\eta^2 \sum_{i \in S^* \setminus S} \mathbb{E}\left[x_i^2\right]\left( \beta^{S^*}_i - \beta^{S}_i \right)^2 +2\eta \mathbb{E}\left[ \left(x^t\left( \beta^{S^*} - \beta^S \right) \right)^2 \right].
	\end{aligned}
	\end{multline*}
	
	%	\begin{align*}
	%	\left(s^* - k\right)\min_{i \in S^*\setminus S} \mathbb{E} \left[ \left( x^t \beta^{S} + \eta  \beta^{S^*}_i x_i - y \right)^2\right]& \le \left(s^* - k\right) \mathbb{E} \left[ \left(x^t\beta^S -y\right)^2\right]  \\
	%	& +\eta^2 \mathbb{E}\left[x_i^2\right]\sum_{i \in S^* \setminus S} \left( \beta^{S^*}_i - \beta^{S}_i \right)^2 +2\eta \mathbb{E}\left[ \left(x^t\left( \beta^{S^*} - \beta^S \right) \right)^2 \right].
	%	\end{align*}
	
	Optimizing over $\eta$ we obtain:
	\begin{equation*}\label{ineq:lem}
	\min_{\eta \in \mbr,i \in S^*\setminus S} \mathbb{E} \left[ \left( x^t \beta^{S} + \eta \left( \beta^{S^*}_i - \beta^{S}_i\right)x_i - y \right)^2\right] \le \mathbb{E} \left[ \left( x^t\beta^S - y\right)^2\right] -\frac{1}{s^*-k}\frac{\mathbb{E}\left[\left(x^t\left( \beta^{S^*} - \beta^S \right) \right)^2\right]^2}{ \sum_{i \in S^*}\mathbb{E}\left[x_i^2\right] \left( \beta^{S^*}_i - \beta^{S}_i\right)^2}.
	\end{equation*} 
	Observe that: $\mathbb{E}\left[ \left(x^t \left(\beta^{S^*} - \beta^S\right)\right)^2\right] = \norm{\Sigma_{S^*}^{1/2} \left( \beta^{S^*}- \beta^{S}\right) }_2^2 \ge \rho \norm{\beta^{S^*}- \beta^{S}}_2^2$. Moreover, $\mathbb{E}\left[x_i^2\right] \le L$. We plug in this inequality into the above
	and obtain the announced conclusion. 
\end{proof}

Now we prove Lemma~\ref{lem:ord}. Using Lemma~\ref{lem:ord1} we have:
\begin{equation*}
\inf_{\alpha \in \mathbb{R}, i \in S^*\setminus S} \mathbb{E} \left[ \left( x^t \beta^{S} + \alpha \beta^{S^*}_i x_i - y \right)^2\right] = \mathbb{E} \left[ (y-x^t\beta^S)^2\right] - \max_{i \in S^*\setminus S} \frac{1}{\left(\beta_i^{S^*}\right)^2\mathbb{E}\left[x_i^2\right]}\mathbb{E}\left[\beta_i^{S^*}x_i \left(x^t\beta^{S}-y\right)\right]^2,
\end{equation*} 
which is equivalent to:
\begin{equation*}
\max_{i \in S^*\setminus S} \frac{1}{\sqrt{\mathbb{E}\left[x_i^2\right]}}\mathbb{E}\left[x_i \left(x^t\beta^{S}-y\right)\right] = \left( \mathbb{E} \left[ (y-x^t\beta^S)^2\right] - \inf_{\alpha \in \mathbb{R}, i \in S^*\setminus S} \mathbb{E} \left[ \left( x^t \beta^{S} + \alpha \left( \beta^{S^*}_i - \beta^{S}_i\right)x_i - y \right)^2\right]   \right)^{1/2}  
\end{equation*}
Using Lemma~\ref{lem:ord2}, we have:
\begin{equation}\label{eq:lem_z}
\max_{i \in S^*\setminus S} \frac{1}{\sqrt{\mathbb{E}\left[x_i^2\right]}}\mathbb{E}\left[x_i \left(x^t\beta^{S}-y\right)\right] \ge \left(  \frac{1}{s^* - k} \frac{\rho}{L} \mathbb{E}\left[ \paren[1]{x^t(\beta^{S^*} - \beta^S)}^2\right]  \right)^{1/2}.
\end{equation}
Now we use Claim~\ref{cl:add} and inequality~\eqref{eq:lem_z}:
\begin{align*}
\max_{i \in S^*\setminus S} \mathbb{E}\left[x_i \left(x^t\beta^{S}-y\right)\right] &\ge  \max_{i \in S^*\setminus S} \sqrt{\frac{\rho}{\mathbb{E}\left[x_i^2\right]}}\mathbb{E}\left[x_i \left(x^t\beta^{S}-y\right)\right]\\
&\ge \sqrt{\rho}\max_{i \in S^*\setminus S} \frac{1}{\sqrt{\mathbb{E}\left[x_i^2\right]}}\mathbb{E}\left[x_i \left(x^t\beta^{S}-y\right)\right]\\
&\ge \sqrt{\rho} \left( \frac{1}{s^*-k} \frac{\rho}{L}\mathbb{E}\left[ \paren[1]{x^t(\beta^{S}- \beta^{S^*})}^2\right]   \right)^{1/2} \\
&\ge \frac{\rho}{\sqrt{L}} \frac{1}{\sqrt{s^* - k}}  \norm{\Sigma^{1/2}_{S^*} \left( \beta^{S^*} - \beta^{S}\right)}_2 \\
&\ge \sqrt{\frac{\rho^{3}}{L}} \frac{1}{\sqrt{s^* - k}} \norm{\beta^{S^*} - \beta^S}_2.	
\end{align*}
The conclusion follows from the definition $Z_i^S = \mathbb{E} \left[ x_i \left(x^t\beta^S - y\right)\right]$.

\section{Computational Complexity Comparisons}\label{sec:comp_compare}

\subsection{Proof of Corollary~\ref{corollary1}:}

Suppose Assumptions~\ref{ass:ass4}, \ref{ass:ass3}, \ref{ass:ass1} and ~\ref{ass:ass2} hold. Consider the procedure \selectproc given by Algorithm~\ref{algo:step}, \tryselectproc given by Algorithm~\ref{algo:select_mc}, and \optimproc as in Algorithm~\ref{alg:optim}. Assume that $S \subsetneq S^*$ and denote $k:=|S|$. Using the result of theorem~\ref{th:main} we have with probability at least $1-\delta$:

\begin{align*}
C_{\optimproc}^S &\le \kappa k^3 \max\left\lbrace \frac{1}{Z_{i^*}^2} ; \frac{\sqrt{k}}{Z_{i^*}} \right\rbrace \log \left( \frac{\bar{k}}{ \delta\left| Z_{i^*}\right|}\right); \\
C_{\tryselectproc}^S  & \le \kappa d \max\left\lbrace \frac{1}{Z_{i^*}^2} ; \frac{\sqrt{\bar{k}}}{Z_{i^*}} \right \rbrace \log \left( \frac{d}{ \delta\left| Z_{i^*}\right|}\right)\log \left( \frac{\bar{k}}{  \left| Z_{i^*}\right|}\right);\\
\end{align*}
where $\left|Z_{i^*}\right| = \max_{i \in [d]} \left\{ \left|Z_i\right|\right\}$, and $\kappa$ is a constant depending on $\rho, L, M$ and $\mu$ (for which the value may vary from line to line).

We plug-in the inequality of lemma~\ref{lem:ord} and obtain:
\begin{align*}
C_{\optimproc}^S &\le \kappa k^3 \max\left\lbrace \frac{s^* - k}{\norm{\beta^{S^*}_{S^* \setminus S}}_2^2} ; \frac{\sqrt{k(s^* - k)}}{\norm{\beta^{S^*}_{S^* \setminus S}}_2} \right\rbrace \log \left( \frac{\bar{k}}{ \delta \norm{\beta^{S^*}_{S^* \setminus S}}_2}\right); \\
C_{\tryselectproc}^S  & \le \kappa d \max\left\lbrace \frac{s^* - k}{\norm{\beta^{S^*}_{S^* \setminus S}}_2^2} ; \frac{\sqrt{\bar{k(s^* - k)}}}{\norm{\beta^{S^*}_{S^* \setminus S}}_2} \right \rbrace \log \left( \frac{d}{ \delta \norm{\beta^{S^*}_{S^* \setminus S}}_2}\right)\log \left( \frac{\bar{k}}{  \norm{\beta^{S^*}_{S^* \setminus S}}_2}\right);\\
\end{align*}

Hence, using the fact that $\left|S^* \setminus S\right| = s^* - k$ and the definition of $\tilde{\beta}_{(k+1)}$:

\begin{align*}
C_{\optimproc}^S &\le \kappa k^3 \max\left\lbrace \frac{1}{\tilde{\beta}_{(k+1)}^2} ; \frac{\sqrt{k}}{\tilde{\beta}_{(k+1)}} \right\rbrace \log (\frac{\bar{k}}{\delta\tilde{\beta}_{(k+1)}^2}); \\
C_{\tryselectproc}^S  & \le \kappa d \max\left\lbrace \frac{1}{\tilde{\beta}_{(k+1)}^2} ; \frac{\sqrt{k}}{\tilde{\beta}_{(k+1)}} \right\rbrace \log^2 (\frac{\bar{k}}{\delta\tilde{\beta}_{(k+1)}^2});\\
\end{align*}

The following claim concludes the proof:

\begin{claim}
	Under the assumptions of theorem~\ref{th:main}:
	\begin{equation*}
	\tilde{\beta}_{(k+1)}  \le  \frac{1}{\sqrt{\rho s^*}}	
	\end{equation*}
\end{claim}
\begin{proof}
	We have by definition of $\tilde{\beta}_{(k+1)}$:
	\begin{align*}
	\tilde{\beta}_{(k+1)}^2 &= \frac{1}{s^* - k} \sum_{i = k+1}^{s^*} \beta^2_{(i)}\\ 
	&\le \frac{s^* - k}{s^*} \quad \frac{1}{s^* - k} \sum_{i = k+1}^{s^*} \beta^2_{(i)} + \frac{k}{s^*} \quad\frac{1}{k}\sum_{i=1}^{k} \beta^2_{(i)} \\
	& \le \frac{1}{s^*} \sum_{i=1}^{s^*} \beta^2_{(i)} = \frac{1}{\rho s^*}
	\end{align*}
\end{proof}

\subsection{Computational complexity of the Orthogonal Matching Pursuit} 

\label{se:complexity}

We consider OMP (Algorithm~\ref{algo:omp}) as a benchmark and show that OOMP is more efficient in time complexity. OMP was initially derived under the fixed design setting presented below:

%In this section we compare the computational complexity of our procedure OOMP with the Orthogonal Matching Pursuit. Let us state the fixed design framework and present the assumptions and theoretical guarantees of OMP as presented in~\cite{zhang2009consistency}:

Let $\bm X = [x_1,\dots,x_d] \in \mathbb{R}^{n \times d}$ an $n \times d$ data matrix and $\bm Y = [y_1,\dots,y_n]$ a response vector generated according to the sparse model:
\begin{equation*}
\bm Y = \bm X\beta^{S^*} + \bm \epsilon.
\end{equation*} 

Where $\bm \epsilon = [\epsilon_1,\dots,\epsilon_n]$ is a zero mean random noise vector and $\text{support}(\beta^{S^*}) = S^*$. Define the following quantities:
\begin{equation*}
\hat{\mu}_{S^{*}} = \max_{i \notin S^*} \norm{ \left( \bm X^t_{S^*}\bm X_{S^*} \right)^{-1} \bm X_{S^*}^t \bm x_i }_1,
\end{equation*}
and let $\hat{\rho}_{S^{*}}$ be the least eigenvalue of the empirical covariance matrix $\hat{ \bm \Sigma}_{S^*} = \frac{1}{n} \bm X_{S^*}^t \bm X_{S^*}$.

\paragraph{OMP theoretical guarantees}
\begin{assumption}\label{ass:omp}
	Assume that:
	\begin{itemize}
		\item $\hat{\mu}_{S^*} <1$ and $\hat{\rho}_{S^*} >0$.
		\item $\epsilon_i$, for $i\in [1,n]$ are i.i.d random variables bounded by $\sigma$.
	\end{itemize}
\end{assumption}

\begin{theorem}[\cite{zhang2009consistency}]\label{th:omp}
	Consider the OMP procedure (Algorithm~\ref{algo:omp}), suppose Assumption~\ref{ass:omp} holds. Then for all $\delta \in (0,1)$, if the sample size $n$ satisfies:
	\begin{equation}\label{eq:condn}
	n \ge \frac{18 \sigma^2 \log(4d/\delta)}{(1-\hat{\mu}_{S^*})^2 \hat{\rho}_{S^*}^2 \min_{i \in S^*} |\beta^{S^*}_i|^2},
	\end{equation}
	then the output of the procedure Algorithm~\ref{algo:omp} recovers $S = S^*$,
	with probability at least $1-\delta$.	
\end{theorem}

\paragraph{OMP computational complexity:}

We derive the computational complexity of OMP. Consider one iteration of Algorithm~\ref{algo:omp} and denote $k:=|S|$. We assimilate the command:
\begin{equation}\label{com:n1}
i \gets \text{argmax}_{j \notin S}|\bm{X}_{.j}^{t}(\bm{Y}-\bm{X}\bar{\beta})|
\end{equation}  
to \tryselectproc and denote $C_{\tryselectproc,k}^{omp}$ its computational complexity. Moreover, we assimilate the command:
\begin{equation}\label{com:n2}
\bar{\beta} \gets \underset{\text{supp}(\beta) \subseteq S}{\text{argmin}} \|\bm{X}\beta-\bm{Y}\|^{2}
\end{equation} 
to \optimproc and denote $C_{\optimproc,k}^{omp}$ its computational complexity.% Lemma~\ref{lem:compomp} below gives bounds on   $C_{\tryselectproc,k}^{omp}$ and $C_{\optimproc,k}^{omp}$.
We assume the OMP is run with $n^{\text{OMP}}$ prescribed by Theorem~\ref{th:omp} for exact support recovery. We introduce the following additional notation: $a \simeq b$ if there exists numerical constants $c_1$ and $c_2$ such that:
$a \le c_1 b$ and $b \le c_2 a$. 

\begin{lemma}
	Consider Algorithm~\ref{algo:omp} with inputs $(\bm X, \bm Y, \delta)$, and suppose assumption~\ref{ass:omp} holds. Then if $n$ satisfies~\eqref{eq:condn} we have:
	\begin{align*}
	C_{\optimproc,k}^{omp} &\simeq  \frac{\sigma^2 k \log(d/\delta)}{\left(1 - \hat{\mu}_{S^*}\right)^2 \hat{\rho}_{S^*}^2 \min_{i \in S^*} |\beta^{S^*}|^2};\\
	C_{\tryselectproc,k}^{omp} &\simeq \frac{\sigma^2 d \log(d/\delta)}{\left(1 - \hat{\mu}_{S^*}\right)^2 \hat{\rho}_{S^*}^2 \min_{i \in S^*} |\beta^{S^*}|^2}.
	\end{align*}
\end{lemma}

\begin{proof}
	Performing command~\eqref{com:n1} requires computing $\bm X^t \left( \bm Y - \bm X \bar{\beta}\right)$ and selecting the maximum of a list of (at most) $d$ elements, thus $C_{\tryselectproc,k}^{omp} \simeq dn^{\text{OMP}}$.
	Command~\eqref{com:n2} can be performed using a rank one update. Thus: $C_{\optimproc,k}^{omp} \simeq kn^{\text{OMP}}$. To conclude we use Theorem~\ref{th:omp}, which prescribes:
	\begin{equation*}
	n^{\text{OMP}} =  \frac{18 \sigma^2 \log(4d/\delta)}{(1-\hat{\mu}_{S^*})^2 \hat{\rho}_{S^*}^2 \min_{i \in S^*} |\beta^{S^*}_i|^2}.
	\end{equation*} 
\end{proof}

Hence, the computational complexity for full support recovery using OMP satisfies:

\begin{equation}\label{eq:comp}
C^{\text{OMP}} = \mathcal{O}\left( \frac{s^*d\log(d/\delta)}{\min_{i \in S^*} \{(\beta_i^*)^2\}}\right)
\end{equation}

\subsection{SSR computational complexity}
SSR (Streaming Sparse Regression) is an online procedure guaranteed to perform well under similar conditions to the Lasso \cite{steinhardt2014statistics}. Theoretical guarantees show that if the number of iterations is large enough the support recovery is achieved with high probability.

Theorem 8.2 in \cite{steinhardt2014statistics} states that, the output vector $\hat{\beta}_T$ satisfies with probability at least $1-5\delta$, $\text{supp}(\hat{\beta}_T) \subseteq S^*$ and:
\begin{equation}\label{eq:ssr}
\norm{\hat{\beta}_T - \beta^*}^2 = \mathcal{O}\left( \frac{(s^*)^2\log(d\log(T)/\delta)}{T}\right),
\end{equation} 

where we used the bound $B \le 6\sqrt{s^*} \frac{M^2}{\sqrt{\rho}}$.
Hence, a sufficient condition to achieve the full support recovery $\text{supp}(\hat{\beta}_T) = S^*$ is :  $\norm{\hat{\beta}_T - \beta^*}^2 \le \min_{i \in S^*} \{(\beta_i^*)^2\}$. Using \eqref{eq:ssr} leads to the following bound on the number of iterations to recover all the support of $\beta^*$:
\begin{equation*}
T = \mathcal{O} \left( \frac{(s^*)^2\log(d/\delta)}{\min_{i \in S^*} \{(\beta_i^*)^2\}}\right) 
\end{equation*}

One iteration of Algorithm 2 in \cite{steinhardt2014statistics} has a computational complexity of $\mathcal{O}(d)$. Hence, the total computational complexity for full support recovery $C^{\text{SSR}}$ satisfies:
\begin{equation}\label{eq:cssr}
C^{\text{SSR}} = \mathcal{O}\left( \frac{(s^*)^2d\log(d/\delta)}{\min_{i \in S^*} \{(\beta_i^*)^2\}}\right)
\end{equation} 

\subsection{Proof of Corollary~\ref{corollary2}}

Assuming that $d>(s^*)^3$, we have for every $S \subset S^*$: $C^S_{\optimproc} \le C^S_{\tryselectproc}$. Hence, using corollary~\ref{corollary1}, we have:
\begin{equation}\label{eq:coomp}
C^{OOMP} \le \kappa d \sum_{i=1}^{s^*} \frac{1}{\tilde{\beta}^2_{(s^*-i)}} \log\left(\frac{d}{\delta \beta_{(s^*)}^2}\right)\log\left(\frac{s^*}{\beta_{(s^*)}^2}\right)
\end{equation}   

We plug-in the bounds in \eqref{eq:comp} and \eqref{eq:cssr}:
\begin{align}\label{eq:compare_c}
C^{OOMP} &\le \kappa \sum_{i=1}^{s^*} \frac{\beta_{(s^*)}^2}{\tilde{\beta}^2_{(s^*-i)}} \log\left(\frac{d}{\delta \beta_{(s^*)}^2}\right)\log\left(\frac{s^*}{ \beta_{(s^*)}^2}\right)\frac{C^{\text{OMP}}}{s^*\log(d/\delta)}. \\
C^{OOMP} &\le \kappa \sum_{i=1}^{s^*} \frac{\beta_{(s^*)}^2}{\tilde{\beta}^2_{(s^*-i)}} \log\left(\frac{d}{\delta \beta_{(s^*)}^2}\right)\log\left(\frac{s^*}{ \beta_{(s^*)}^2}\right)\frac{C^{\text{SSR}}}{(s^*)^2\log(d/\delta)}. \\
\end{align}

Recall that:
\begin{equation*}
\frac{\log\left(\frac{d}{\delta \beta_{(s^*)}^2}\right) \log\left(\frac{s^*}{ \beta_{(s^*)}^2}\right)}{\log(d/\delta)} \le \log^2\left(\frac{s^*}{ \beta_{(s^*)}^2}\right).
\end{equation*}

We conclude that:
\begin{align*}
\frac{C^{\text{OOMP}}}{C^{\text{OMP}}} &\le  \kappa \log^2\left(\frac{s^*}{ \beta_{(s^*)}^2}\right)\frac{1}{s^*}\sum_{i=1}^{s^*} \frac{\beta_{(s^*)}^2}{\tilde{\beta}_{(i)}^2} \quad C^{\text{OMP}}; \\
\frac{C^{\text{OOMP}}}{C^{\text{SSR}}} &\le  \kappa \log^2\left(\frac{s^*}{ \beta_{(s^*)}^2}\right) \frac{1}{(s^*)^2}\sum_{i=1}^{s^*} \frac{\beta_{(s^*)}^2}{\tilde{\beta}_{(i)}^2} \quad C^{\text{SSR}}; \\
\end{align*}

where $\kappa$ is a constant depending only on $L,M, \rho$ and $\mu$. 

\subsection{A specific scenario: Polynomially decaying coefficients} 

We consider the case where the coefficients of $\beta^{*}$ are given by
\begin{equation}\label{eq:coefs_express}
\beta^*_q = \frac{1}{\sqrt{s^*}} \left(1 - \frac{q-1}{s^*}\right)^\gamma, \quad \text{for } q \in [s^*],
\end{equation}
with $\gamma > 0$. We omit the superscript $*$ to ease notations, in the remainder of this section, all the inequalities and equalities are up to
factors depending only only on $\rho, L, M$ and $\mu$. 

The following lemma provides a bound on the computational complexity of OOMP, OMP and SSR.

\begin{lemma}
	Under the assumptions of Theorem~\ref{th:main}, suppose that $d>(s^*)^3$ and the coefficients of $\beta^*$ are given by \eqref{eq:coefs_express}. Then with probability at least $1-\delta$:
	If $\gamma \neq \frac{1}{2}$:
	\begin{align*}
	C^{\text{OOMP}} &\le \kappa d \left\lbrace \frac{2\gamma \left(2 \gamma +1\right)}{\abs{2\gamma-1}} s^{2\gamma+1} + \frac{2\gamma+1}{\abs{2\gamma-1}} s^2 \right \rbrace \log\left(d/\delta\right)\log\left(s\right) \\
	C^{\text{OMP}} &\simeq  d s^{2\gamma+2} \log(d/\delta)  \\
	\end{align*}
	If $\gamma = \frac{1}{2}$:
	\begin{align*}
	C^{\text{OOMP}} &\le \kappa d  s^2 \log^2(s)\log\left(d/\delta\right) \\
	C^{\text{OMP}} &\simeq  d s^{3} \log(d/\delta)  \\	
	\end{align*} 
\end{lemma}

\begin{proof}
	Recall that $\tilde{\beta}_{(s-k+1)}^2 = \frac{1}{k} \sum_{i=s-k+1}^{s} \beta_i^2$.

	If $\gamma \neq \frac{1}{2}$:
	\begin{align*}
	\sum_{k=0}^{s-1} \frac{1}{\tilde{\beta}_{(s-k)}^2} &= \sum_{k=0}^{s-1} \frac{s-k}{\sum_{q=k+1}^{s}\beta_q^2} \\
	&\le \sum_{k=0}^{s-1} \frac{s-k}{\frac{1}{s} \sum_{q=k+1}^{s} \left(1 - \frac{q-1}{s}\right)^{2\gamma}} \\
	&\le \sum_{k=0}^{s-1} \frac{s^{2\gamma+1}\left(s-k\right)}{\sum_{q=1}^{s-k}q^{2\gamma}}\\
	&\le \sum_{k=0}^{s-1} \frac{s^{2\gamma+1}(s-k)}{\frac{1}{2\gamma +1}(s-k)^{2\gamma+1}} \\
	&\le (2\gamma+1) \sum_{k=0}^{s-1} \frac{s^{2\gamma+1}}{(s-k)^{2\gamma}}\\
	&\le (2\gamma+1) s \sum_{k=0}^{s-1} \left(1 - \frac{k}{s}\right)^{-2\gamma}\\
	&\le (2\gamma+1) s^{2} \left( \frac{1}{s} \sum_{k=0}^{s-2} \left(1 - \frac{k}{s}\right)^{-2\gamma} +  s^{2\gamma-1}\right)\\
	&\le (2\gamma+1) s^{2} \left( \frac{1}{2\gamma-1} \left(\frac{1}{s^{1-2\gamma}}-1\right) + s^{2\gamma-1}\right).
	\end{align*}
	If $\gamma = \frac{1}{2}$:
	
	\begin{align*}
	\sum_{k=0}^{s-1} \frac{1}{\tilde{\beta}_{(s-k)}^2} &= \sum_{k=0}^{s-1} \frac{s-k}{\sum_{q=k+1}^{s}\beta_q^2}\\
	&\le \sum_{k=0}^{s-1} \frac{s-k}{\frac{1}{s} \sum_{q=k+1}^{s} \left(1 - \frac{q-1}{s}\right)} \\
	&\le \sum_{k=0}^{s-1} \frac{s^{2}\left(s-k\right)}{\sum_{q=1}^{s-k}q}\\
	&\le \sum_{k=0}^{s-1} \frac{s^{2}(s-k)}{\frac{1}{2}(s-k)^{2}} \\
	&\le 2 \sum_{k=0}^{s-1} \frac{s^{2}}{(s-k)}\\
	&\le s^{2} \log\left(s\right),
	\end{align*}
	which gives the result.
	
\end{proof}

Using the lemma above, we conclude that, if $d>(s^*)^3$:
\begin{equation*}
\frac{C^{\text{OOMP}}}{C^{\text{OMP}}} \le \kappa \frac{\log^2(s)}{s^{\min\{2\gamma, 1\}}}
\end{equation*}

\end{document}